\documentclass[twoside]{article}

\usepackage[accepted]{aistats2022}
\usepackage[round]{natbib}

\usepackage{mystyle}

\begin{document}

%
\runningtitle{A Last Switch Dependent Analysis of Satiation and Seasonality in Bandits}

%
\runningauthor{Pierre Laforgue, Giulia Clerici, Nicolò Cesa-Bianchi, Ran Gilad-Bachrach}

\twocolumn[

\aistatstitle{A Last Switch Dependent Analysis of\\ Satiation and Seasonality in Bandits}

\aistatsauthor{ Pierre Laforgue\,$^1$ \And Giulia Clerici\,$^1$\hspace{1cm} \And  Nicolò Cesa-Bianchi\,$^1$\hspace{1cm} \And Ran Gilad-Bachrach\,$^{2, 3, 4}$ }

\medskip

\aistatsaddress{ $^1$ Dept.\ of Computer Science, Università degli Studi di Milano, Milan, Italy\\
$^2$ Department of Bio-Medical Engineering, Tel-Aviv University, Israel\\
$^3$ Edmond J. Safra Center for Bioinformatics, $^4$ Sagol School of Neuroscience} ]

\begin{abstract}
Motivated by the fact that humans like some level of unpredictability or novelty, and might therefore get quickly bored when interacting with a stationary policy, we introduce a novel non-stationary bandit problem, where the expected reward of an arm is fully determined by the time elapsed since the arm last took part in a switch of actions.
Our model generalizes previous notions of delay-dependent rewards, and also relaxes most assumptions on the reward function.
This enables the modeling of phenomena such as progressive satiation and periodic behaviours.
Building upon the Combinatorial Semi-Bandits (CSB) framework, we design an algorithm and prove a bound on its regret with respect to the optimal non-stationary policy (which is NP-hard to compute).
Similarly to previous works, our regret analysis is based on defining and solving an appropriate trade-off between \mbox{approximation} and estimation.
Preliminary experiments confirm the superiority of our algorithm over both the oracle greedy approach and a vanilla CSB solver.
\end{abstract}

\section{INTRODUCTION}

As the range of applications of multi-armed bandits increases, new algorithms able to deal with a wider range of phenomena must be designed and analyzed.
When interacting with humans, one can often observe behaviours such as satiation and seasonality, that are hardly captured by stationary policies.
Satiation typically occurs when making recommendations, e.g., for books, music, or movies \citep{kunaver2017diversity,leqi2020rebounding}.
Stationary policies, which learn to recommend a single cuisine or musical genre, fail to capture the human desire for novelty \citep{dimitrijevic1972habituation,kovacs2018rotating}.
Another widespread type of nonstationary behaviour is seasonality.
For example, certain music genres may be preferred during working hours, whereas chill playlists could be more appreciated on evenings or weekends \citep{schedl2018current}.
Nonstationary behaviours also naturally \mbox{occur} in the medical domain.
Consider people suffering from various mental disorders, such as post-traumatic stress disorder, depression, or anxiety.
Some \mbox{remedies} \mbox{include} micro-interventions (mindfulness, positive psychology exercises), Cognitive Behavioural Therapy, or Dialectical Behavioral Therapy \citep{meinlschmidt2016smartphone,owen2018va,schroeder2018pocket,fuller2019randomized}.
However, studies have shown that a diminishing return effect exists when the same intervention is applied repeatedly \citep{paredes2014poptherapy}, motivating the modelling of satiation effects.

In this work, we introduce the first bandit model which captures nonstationary phenomena that include satiation and seasonal effects.
Unlike previous works, where the state structure is fully determined by how long ago each arm was last played, our notion of state keeps track of the time elapsed since an action took part in a switch (i.e., when the arm being pulled changes).
This allows to modulate satiation effects at a level of detail not within reach of previous models.
In \mbox{addition}, our model drops many assumptions on the shape of the expected reward function (such as concavity, Lipschitz continuity, monotonicity), while only retaining boundedness which is plausible in human interactions.
Dispensing with monotonicity enables the modelling of periodicity, which is key to capture seasonality.

\paragraph{Technical contributions.}
The backbone of our analysis goes along the lines of previous works: (1) we show that computing the optimal policy (with value OPT) in our setting is NP-hard; (2) we prove that OPT is well approximated by a certain class of simple policies; (3) we show how to learn the best policy in the approximating class.
Also similarly to previous works, we initially consider cyclic policies (with bounded block length) as approximating class.
One of the main technical hurdles when learning a cyclic policy in a nonstationary setting is the calibration problem.
Namely, the expected reward of the cycle block (i.e., the sequence of arm pulls that is being repeated in the cycle) depends on the current state of all the arms appearing in the block. A simple, yet impractical solution, is to force the learning algorithm to play every block twice in a row, where the first play of a block calibrates the reward estimates computed in the second play. Our solution is radically different: rather than forcing the algorithm to artificially play additional arms, we calibrate an arm by referring to the first time the arm is pulled in the block. These calibration pulls are not used to compute reward estimates, because their reward depends on the last time the arms took part in a switch among the blocks previously played by the learning algorithm. By ignoring calibration pulls, we underestimate the block reward by at most $K$ (the number of arms), which in typical instances is much smaller than the block length.

Our calibration approach reduces the problem of learning the best calibrated block to that of solving an instance of Combinatorial Semi-Bandits (CSB), which we can solve using algorithms whose regret is well understood. As both the regret bound and the approximation factor for OPT depend on the block length $d$, we can then choose $d$ to optimize our performance, thus obtaining a regret bound against OPT of order $KT^{3/4}$, ignoring log factors. Note that this does imply polynomial-time convergence to OPT, as solving the decoding problem (an intermediate step in the regret minimization procedure) for our instance of CSB requires solving an integer linear program, which is NP-hard in general. 
The decoding problem amounts to compute the reward-maximizing block given the current estimates of the state-value function (which we represent with a $K \times d$ table of reward estimates). In practice, we can approximately solve it through a branch-and-bound approach using a LP relaxation to bound the value of the objective.
We implement an approximate solver based on this approach and run an experiment on a simple instance of LSD where our algorithm is seen to outperform two natural baselines.

\paragraph{Related works.}
Previous works addressed various extensions of bandits where rewards depend on past pulls. These include rested models \citep{gittins1979bandit}, such as Rotting Bandits \citep{bouneffouf2016multi,heidari2016tight,cortes2017discrepancy,levine2017rotting,warlop2018fighting,seznec2019rotting} and restless models \citep{whittle1988restless,tekin2012online}.
Similarly to \citet{cella2020stochastic}, our framework is neither rested nor restless, because the reward of arms that are pulled changes differently from the reward of arms that are not pulled.
Unlike \citet{cella2020stochastic}, we do not assume a specific form for the reward function and so we can model satiation and seasonalities.
\citet{kleinberg2018recharging} also look at a similar model, but ---unlike us--- they assume the reward functions to be concave, increasing, and Lipschitz.
Exploiting concavity (but not Lipschitzness) they prove that for every $0 < \varepsilon < 1$ there exists a periodic schedule of length $T \ge K/\varepsilon$ whose asymptotic average reward is at least $(1-\varepsilon)\mathrm{OPT}$.
\citet{pike2019recovering} also assume that the expected reward is a function of the time since the arm is last pulled, but consider reward functions drawn from a Gaussian Process with known kernel.
Another relevant work is \citet{simchi2021dynamic}, where the authors explore the possibility of pulling and collecting delay-dependent rewards from more than one arm at each time step.
Finally, note that \citet{basu2019blocking} investigate an interesting related variant where an arm becomes unavailable for a certain amount of time after each pull.
Although semantically close, we highlight that bandits with delayed feedback \citep{pike2018bandits} are unrelated to our setting.

\begin{figure*}[!t]
    \centering
    \includegraphics[width=0.88\columnwidth]{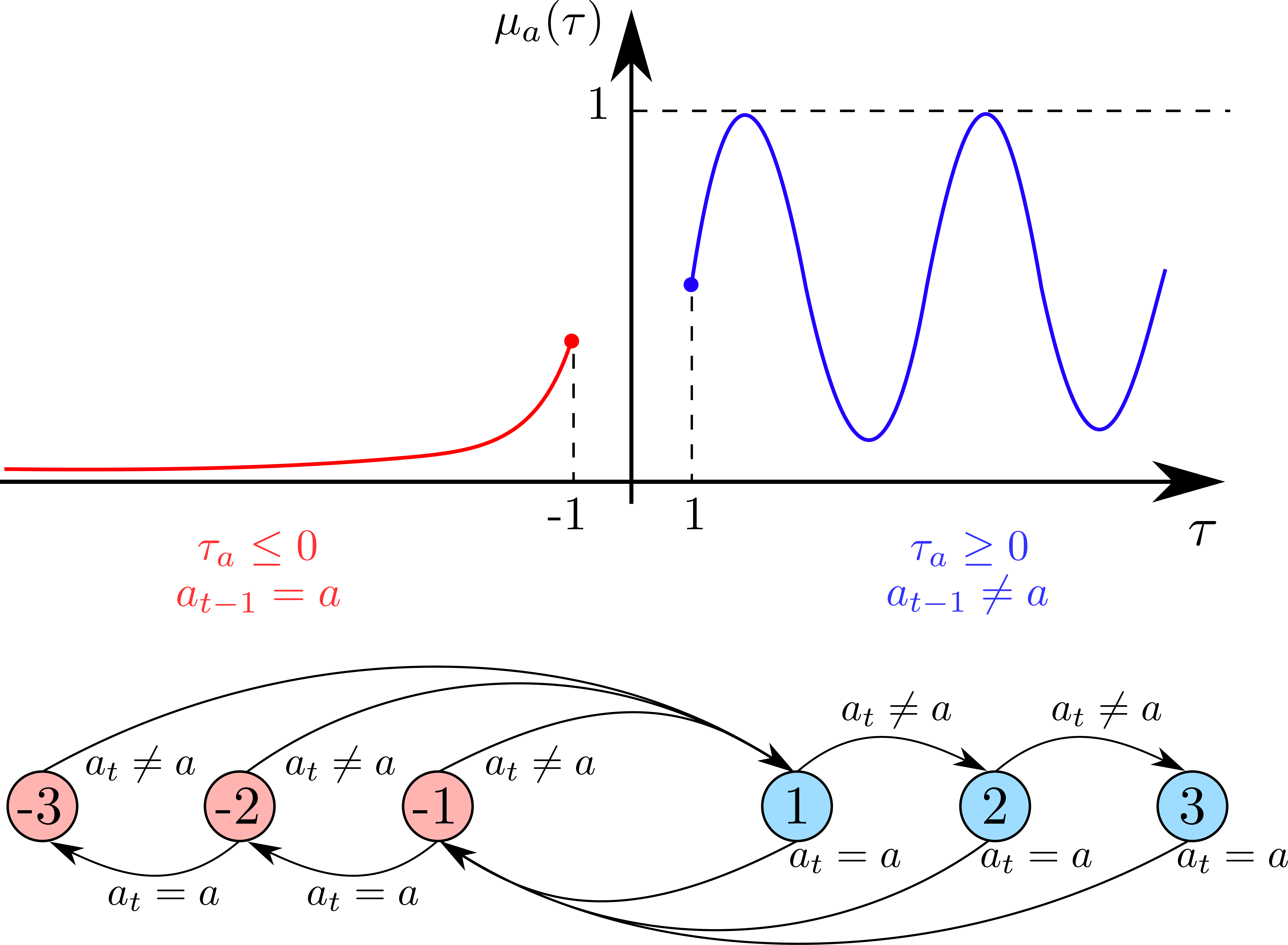}
    \qquad\quad
    \includegraphics[width=0.88\columnwidth]{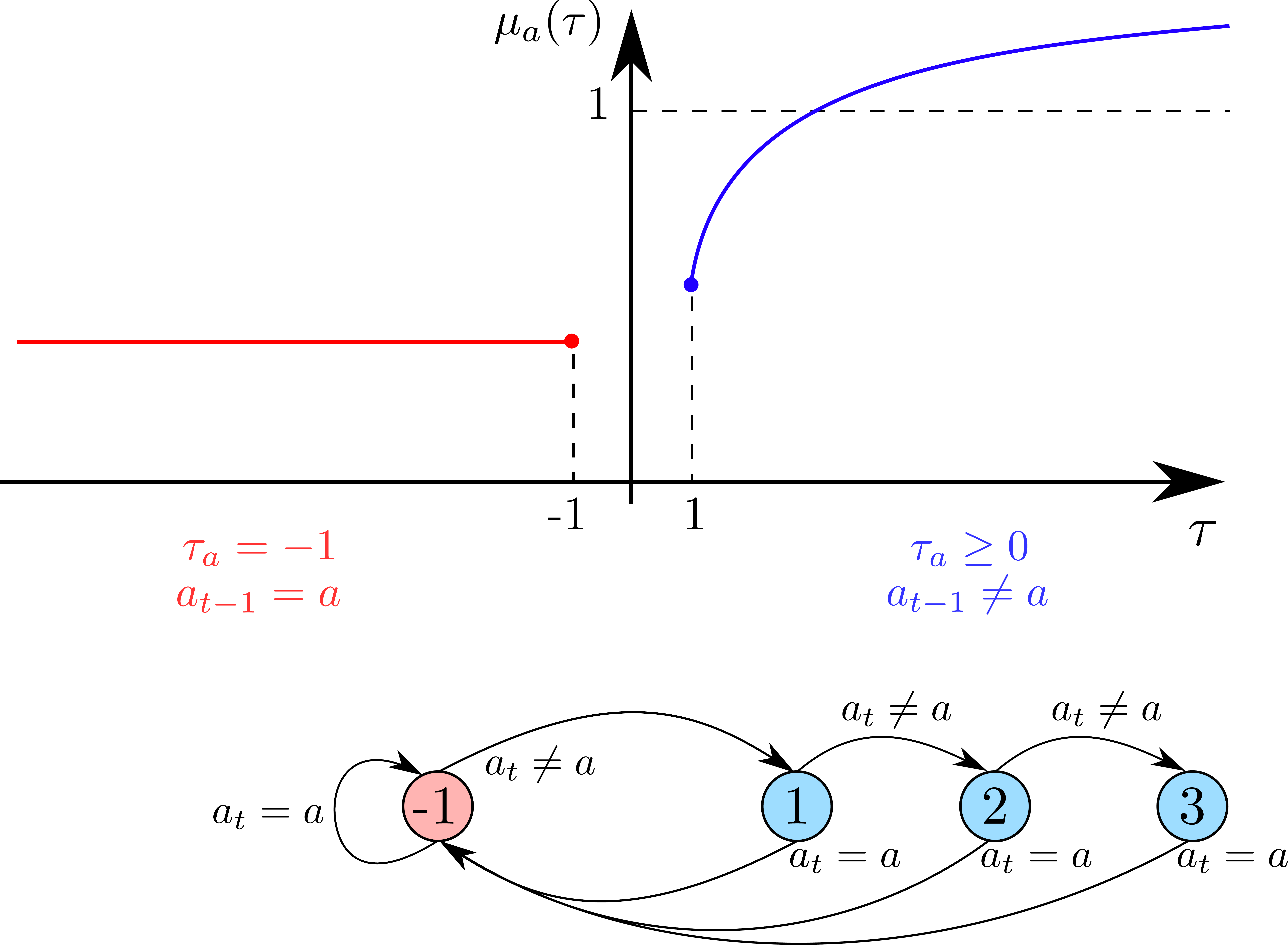}
    \caption{Transition mechanisms (reduced to $6$ and $4$ states), and examples of expected reward function $\mu_a$ for LSD Bandits (left) and models with delay-dependent rewards \citep[e.g.]{kleinberg2018recharging} (right).}
    \label{fig:models}
\end{figure*}

\section{LSD BANDITS}

Let $\mathcal{A}$ be an action set, with cardinality $K = |\A|$.
When necessary, individual actions are distinguished using superscripts, $\A = \{a^{(1)}, \ldots, a^{(K)}\}$, whereas subscripts are used for the temporal indexing, so that $a_t$ refers to the action played at time step $t$.
In our setting, the expected reward of an arm $a$ is assumed to be fully determined by the last time $t$ it took part in a switch. That is, the last $t$ when the tuple $(a_{t-1}, a_t)$ contains $a$ exactly once.
Therefore, every arm $a \in \A$ is endowed with a state $\tau_a$ ($\tau_a$ actually depends on $t$, and we should use $\tau_a(t)$, but we drop this dependence whenever it is understood from the context) such that:
\begin{equation}\label{eq:tau_init}
\tau_a(0) = 1 \text{\quad and \quad} \tau_a(t+1) = \delta_a\big(\tau_a(t), a_t\big)\,,
\end{equation}
where $\delta_a$ is the transition function given by:
\begin{equation}\label{eq:transition}
\delta_a(\tau, a^\prime) = \begin{cases}
\tau - 1        & \text{if} ~~ a^\prime = a, ~~ \tau \le 0\\[0.1cm]
-1                & \text{if} ~~ a^\prime = a, ~~ \tau \ge 0\\[0.1cm]
\hspace{0.26cm} 1 & \text{if} ~~ a^\prime \ne a, ~~ \tau \le 0\\[0.1cm]
\tau + 1        & \text{if} ~~ a^\prime \ne a, ~~ \tau \ge 0
\end{cases}
\end{equation}
The transition graph (reduced to $6$ states) is pictured in \Cref{fig:models} (bottom left).
For each arm $a \in \A$, the state $\tau_a$ keeps track of the last time a switch of actions involved action $a$.
A positive $\tau_a$ means that the last switch involving arm $a$ was $[a,\text{not }a]$, and occurred $\tau_a$ time steps ago.
In other words, arm $a$ has not been played for the last $\tau_a$ rounds.
A negative $\tau_a$ means that the last switch involving arm $a$ was $[\text{not }a, a]$, and occurred $|\tau_a|$ time steps ago.
In other words, arm $a$ has been consistently played for the last $|\tau_a|$ rounds.
We can now define Last Switch Dependent (LSD) Bandits, in which the expected reward of an arm $a$ only depends on its last switch state $\tau_a$.

\begin{definition}[LSD Bandit]\label{def:LSD}
A stochastic bandit with action set $\A$ is a LSD bandit if for every action $a \in \A$ there exists an (unknown) function $\mu_a \colon \mathbb{Z} \rightarrow [0, 1]$, nondecreasing on $\mathbb{Z}^-$, such that the expected reward of arm $a$ is given by $\mu_a(\tau_a)$, where $\tau_a$ is the \emph{last switch} state of $a$, as defined in \eqref{eq:tau_init} and \eqref{eq:transition}.
\end{definition}

Note that a LSD bandit is fully characterized by its reward functions $(\mu_a)_{a \in \A}$ and the noise distribution.
In the rest of the paper, we also use the following shortcut notation: $\mu_i = \mu_{a^{(i)}}$, $\tau_i = \tau_{a^{(i)}}$, and $\btau(t) = (\tau_a(t))_{a \in \A}$.
Given a horizon~$T$, the learner interacts with a LSD bandit as follows.
First, the arm states are initialized to $1$\footnote{In case of a nondecreasing $\mu_i$ on $\mathbb{Z}^+$, the initialization $\tau_i(0) = +\infty$ would be the most sensible. In the absence of monotonicity however, all positive states become equivalent choices for the initialization, $\tau_i(0) = 1$ is one of them.}.
Then, for all time steps $t$ from $1$ to $T$:
\begin{enumerate}[nosep]
\item the learner chooses an action $a_t \in \A$,
\item the learner obtains a stochastic reward $X_t$ with expected value $r_t \coloneqq \mu_{a_t}\big(\tau_{a_t}(t)\big)$,
\item states update: $\forall a \in \A$, $\tau_a(t+1) =
\delta_a\big(\tau_a(t), a_t\big)$.
\end{enumerate}
The goal is to maximize the expected sum of rewards, and the performance is measured through the regret
\[
R_T = \sum_{t=1}^T \mu_{a^*_t}\big(\tau_{a^*_t}(t)\big) - \mathbb{E}\left[\sum_{t=1}^T \mu_{a_t}\big(\tau_{a_t}(t)\big)\right]\,,
\]
where $[a^*_1 \ldots a^*_T]$ is the optimal sequence of actions, i.e., the sequence maximizing the expected sum of rewards over steps $1,\ldots,T$.
LSD Bandits generalize bandits with delay-dependent rewards \citep{kleinberg2018recharging,pike2019recovering,cella2020stochastic} in two ways.

First, we emphasize that the notion of last switch generalizes the notion of delay.
Whereas both definitions coincide on the positive values, last switches also allow to model progressive satiation, while delays cannot.
Indeed, in the case where the same arm is pulled consecutively, the last switch keeps updating towards the negative values, see \eqref{eq:transition}, thus allowing to associate different rewards with different numbers of consecutive pulls.
On the contrary, no matter if the arm is pulled two or ten times in a row, the delay remains equal to $0$ (or $-1$, depending on the convention), as only the fact that the same arm was pulled in the previous round matters.
In other words, a bandit with delay-dependent rewards is a special case of LSD, where the reward function $\mu_a$ is constant over the negative $\tau$'s, while for a general LSD bandit it is only assumed to be nondecreasing---see \Cref{fig:models}, top.
Assuming $\mu_a$ to be nondecreasing on $\mathbb{Z}^-$ can be interpreted as a diminishing return requirement:  the more the same arm is consecutively pulled, the smaller its reward gets (before resetting to the ``default'' value $\mu_a(1)$ when the series is interrupted).
This is natural in most scenarios, and key to derive nonvacuous approximations (\Cref{tab:approx}).

The second generalization regards the assumptions made on the values of $\mu_a$ on $\mathbb{Z}^+$.
In LSD bandits, $\mu_a$ is only supposed to be bounded.
This assumption encompasses the framework of \cite{cella2020stochastic}, who considers a parametric class of bounded expected reward functions that are increasing and have bounded memory (their maximum is attained after a certain arm-dependent delay).
The Recharging Bandit model of \citet{kleinberg2018recharging}, instead, assumes expected reward functions that are concave, increasing, Lipschitz continuous, but not bounded, making their framework and ours incomparable.
We nonetheless highlight several limitations of their setting: (1) concavity excludes simple examples, such as $\mu_a(\tau) = \mathbb{I}\{\tau \ge t_a\}$ for some predefined \mbox{threshold} $t_a$, (2) concavity prevents from modelling series of arm pulls, as a negative branch would force rewards to decrease unbounded, and (3) monotonicity prevents from modelling seasonality, see \Cref{fig:models}, top.
We finally highlight that relaxing the monotonicity of $\mu_a$ on $\mathbb{Z}^+$ unlocks many more applications than just modelling seasonality.
Note that seasonality is understood with respect to the sequence of arm pulls.
This matches the absolute time seasonality as soon as the learner is asked to make predictions at regular intervals.

Relaxing the concavity assumption has also important consequences on the difficulty of the problems which can be considered.
For concave reward functions the oracle greedy policy provides a $1/2$-approximation of the optimal policy  \citep{kleinberg2018recharging}, yet \Cref{ex:anti-kleinberg} shows that oracle greedy can be arbitrarily bad in our setting, highlighting that LSD Bandits can be more difficult than Recharging Bandits.
\begin{example}\label{ex:anti-kleinberg}
Consider the LSD bandit defined by the following reward functions: $\mu_1(\tau) = \epsilon + (1-\epsilon)\mathbb{I}\{\tau \ge 1\}$, and $\mu_2(\tau) = 0$ $\forall\tau$.
An oracle greedy strategy, which pulls at each time step the arm with highest expected reward (assuming the knowledge of the $\mu_a$) would always pull arm $a^{(1)}$ obtaining a reward of $1 + (T-1)\epsilon$ after $T$ rounds.
Instead, the optimal policy alternates between arms $a^{(1)}$ and $a^{(2)}$ and gets a $T/2$ overall reward.
By making $\epsilon$ arbitrary close to $0$, we can thus make oracle greedy arbitrarily bad.
We conclude with a remark on why this example would be ruled out by concavity.
In \cite{kleinberg2018recharging}, concavity is actually defined with respect to the origin, such that the reward obtained in case of consecutive pulls (here it is $\epsilon$) is considered as an increment from $0$.
Concavity then prevents the next increments from being bigger, while here it is equal to $1 - \epsilon$, which is greater than $\epsilon$ as soon as $\epsilon < 1/2$.
And for $\epsilon \ge 1/2$, one can note that oracle greedy is indeed a $1/2$ approximation of the optimal strategy, as revealed by the above computations.
\end{example}


\subsection{Hardness Results}
\label{subsec:hard}

Going beyond \Cref{ex:anti-kleinberg}, we emphasize the difficulty of solving LSD bandits by showing several hardness results.
All missing proofs are deferred to \Cref{apx:proofs}.
\begin{proposition}\label{prop:hardness}
Computing the optimal policy for LSD bandits is NP-hard.
\end{proposition}
The proof builds on \citep[Theorem 3.1]{basu2019blocking}, and is detailed in \Cref{apx:hardness}.
Since computing the global optimal policy is NP-hard even with full knowledge of the problem instance, a common approach then consists in learning the optimal policy within a simpler class of approximating policies.
For LSD bandits, a natural approximating class is the set of policies with cyclic play sequences. A play sequence $a_1, a_2, \ldots$ is cyclic if there exist $t_0$ and $d>0$ such that $\forall t\geq t_0$ it holds that $a_{t+d} = a_t$.
Cyclic policies can be shown to be optimal for 2-armed LSD bandits (\Cref{lem:cyclic}), and give good constant factor approximations in general (\Cref{prop:approx_cyclic}).
\begin{lemma}\label{lem:cyclic}
For a LSD bandit with two arms, any \mbox{deterministic} policy induces a sequence of pulls which is cyclic from a certain time step onward.
This holds in particular for the optimal deterministic policy.
\end{lemma}
With $3$ arms however, we can exhibit optimal policies which do not induce cyclic sequences.
Note that this does not imply that no \mbox{optimal} policy is cyclic.
Answering this question is however quite involved, and deferred to future work.
In \Cref{prop:approx_cyclic}, we show a slightly weaker result, namely that cyclic policies can almost reach the best average reward, i.e., up to a term which is inversely proportional to the cycle length.
\begin{proposition}\label{prop:approx_cyclic}
Let $(\mu_i)_{i=1}^K$ be an LSD bandit with $K$ arms and constant expected rewards on $\mathbb{Z}^-$.
Let $T \ge 0$ be the horizon, and $d \ge 0$ that divides $T$.
Then, there exists a cyclic policy $\pi$ with cycle length $d$ (and $t_0 = 1$) such that
\begin{equation}\label{eq:cyclic_lower_bound}
\frac{1}{T} \sum_{t=1}^T r_\pi(t) \ge \frac{1}{T} \sum_{t=1}^T r^*(t) - \frac{K}{d}
\end{equation}
where $r_\pi(t)$ and $r^*(t)$ are the expected rewards obtained at time step $t$ by $\pi$ and the optimal policy, respectively.
When the $\mu_i$ are not constant on $\mathbb{Z}^-$, \eqref{eq:cyclic_lower_bound} holds with $K+2$ instead of $K$ in the right-hand side.
Note also that \eqref{eq:cyclic_lower_bound} requires $d \ge K$ (respectively $ d \ge K+2$) to be nonvacuous, as we have: $0 \le r^*(t) \le 1$.
\end{proposition}

\Cref{prop:approx_cyclic} is proved at the end of the section for constant $\mu_i$ on $\mathbb{Z}^-$.
Thus, looking for the best cyclic policy simplifies the problem (as the search space is reduced), while maintaining a control on the approximation error via the cycle length.
Unfortunately, this strategy is not viable either: it can be shown that finding the optimal cyclic policy for LSD bandits remains NP-hard, even when arms are totally ordered.
\begin{proposition}\label{prop:hardness_cyclic}
Finding the optimal cyclic policy for LSD bandit problems is NP-hard, even with separable reward functions of the form $\mu_i(\tau) = \mu_i^0 \cdot \gamma(\tau)$.
\end{proposition}
We conclude this section by proving \Cref{prop:approx_cyclic}.
For $B = [a_1 \ldots a_d]$, let $r(B \mid \btauinit) = \sum_{t=1}^d \mu_{a_t}(\tau_{a_t}(t))$ be the sum of the expected rewards obtained by playing the actions in $B$ when $\btau$ is initialized to $\btau(1) = \btauinit$.
\Cref{prop:approx_cyclic} is essentially derived from the following key lemma, which exploits the boundedness assumption to relate $r(B \mid \btau)$ for different initial states.
\begin{lemma}\label{lem:key}
Let $(\mu_i)_{i=1}^K$ be an LSD bandit with $K$ arms and constant expected rewards on $\mathbb{Z}^-$.
Then, for all block $B$ and any initial states $\btauinit, \btauinit' \in \mathbb{Z}^K$, we have
\begin{equation}\label{eq:key_constant}
r(B \mid \btauinit') \ge r(B \mid \btauinit) - K~.
\end{equation}
If the $\mu_i$ are not constant on $\mathbb{Z}^-$, a slight modification of $B$ in the left-hand side is required to maintain a similar bound.
Formally, for any block $B$ of length $d$, there exists a block $B'$ of length $d$ such that for any initial states $\btauinit, \btauinit'$, we have
\begin{equation}\label{eq:key_general}
r(B' \mid \btauinit') \ge r(B \mid \btauinit) - (K + 2)~.
\end{equation}
\end{lemma}
\begin{proof}
For simplicity, here we only prove the special case where the $\mu_i$ are constant on $\mathbb{Z}^-$.
First note that after the first play of an arm $a$, the switch $[a, \text{not }a]$ occurs, and the arm goes to $\tau_a = 1$, independently of its initial state.
Now, what happens during the first play of $a$?
Let $t_a$ be the time step at which arm $a$ is played for the first time.
Even if there are many consecutive pulls of $a$, the rewards collected for $\btauinit$ and $\btauinit'$ are $\mu_a(\tau'_{\mathrm{init},a} + t_a - 1), \mu_a(-1), \ldots, \mu_a(-1)$ and $\mu_a(\tau_{\mathrm{init},a} + t_a - 1), \mu_a(-1), \ldots, \mu_a(-1)$ because the $\mu_i$ are constant on $\mathbb{Z}^-$.
Thus, the only difference is $\mu_a(\tau'_{\mathrm{init},a} + t_a - 1) - \mu_a(\tau_{\mathrm{init},a} + t_a - 1) \le 1$.
Over the $K$ arms, the total difference cannot exceed $K$, giving \eqref{eq:key_constant}.
Note that in full generality, $K$ could be replaced by the number of different arms played in $B$.
When the $\mu_i$ are not constant on $\mathbb{Z}^-$, a finer comparison of the first pulls is required.
\end{proof}

\begin{table}
\centering
\begin{tabular}{|r|*{3}{c|}}\hline
$\mu_a$\hspace{0.4cm} & \---- on $\mathbb{Z}^-$ & $\nearrow$ on $\mathbb{Z}^-$ & $\sim$ on $\mathbb{Z}^-$ \\\hline
$\nearrow$ on $\mathbb{Z}^+$ & \cellcolor{Gray} $-(K-1)$ & \cellcolor{blue!20!white} $-(K+1)$ & \cellcolor{red!20!white} $-d$ \\\hline
$\sim$ on $\mathbb{Z}^+$ & \cellcolor{blue!20!white} $-K$ & \cellcolor{blue!20!white} $-(K+2)$ & \cellcolor{red!20!white} $-d$ \\\hline
\end{tabular}
\caption{Approximation errors by block of size $d$ when $\mu_a$ is constant (\----), nondecreasing ($\nearrow$), or non-monotone ($\sim$) on $\mathbb{Z}^-$ and $\mathbb{Z}^+$. The grey cell represents previous works. Our setting covers the blue ones. The red cells indicate vacuous approximations.}
\label{tab:approx}
\vspace{-0.3cm}
\end{table}

The approximation errors by block of size $d$ that can be achieved depending on the assumptions made on $\mu_a$ are summarized in \Cref{tab:approx}.
Relaxing the assumptions on $\mathbb{Z}^+$ comes with a low price, while the monotonicity on $\mathbb{Z}^-$ is critical.
We now provide a proof sketch of \Cref{prop:approx_cyclic} when the $\mu_i$ are constant on $\mathbb{Z}^-$, and show an example where our analysis is essentially tight.
\par{\it Proof sketch of \Cref{prop:approx_cyclic}.}{
Note that there exists a time step $t_0 \le T - d + 1$ such that $(1/d)\sum_{t=t_0}^{t_0 + d - 1}r^*(t) \ge (1/T)\sum_{t=1}^{T}r^*(t)$.
Let $B^* = [a^*_{t_0} \ldots a^*_{t_0 + d -1}]$, and $\btau^*(t)$ be the sequence of states generated by the optimal policy, such that $\sum_{t=t_0}^{t_0 + d - 1}r^*(t) = r(B^* \mid \btau^*(t_0))$.
The policy $\pi$ consists in playing repeatedly block $B^*$.
However, except for the first $d$ steps, $B^*$ is played by $\pi$ with an initial state $\btau_{B^*}$, i.e., the state reached by the system after a play of $B^*$, which might be different from $\btau^*(t_0)$.
Applying \Cref{lem:key}, and assuming for simplicity that states were initialized in $\pi$ to $\btau_{B^*}$, we obtain
\begin{align*}
\frac{1}{T} \sum_{t=1}^T r_\pi(t) &= \frac{r(B^* \mid \btau_{B^*})}{d} \ge \frac{r(B^* \mid \btau^*(t_0)) - K}{d}\\
&\ge \frac{1}{T} \sum_{t=1}^T r^*(t) - \frac{K}{d}~.
\end{align*}
If the $\mu_i$ are not constant on $\mathbb{Z}^-$, the repeated block is built using the second part of \Cref{lem:key}.
\qed}

\begin{example}\label{ex:tight}
Consider the LSD bandit defined by the following reward functions
\[
\mu_i(\tau) = \mathbb{I}\{\tau \ge K - 1\} \qquad \forall i \le K.
\]
The optimal policy consists in repeating the block $[a^{(1)} \ldots a^{(K)}]$ and obtains an average reward of $1$.
Let $d = 2K - 1$, such that (up to permutations) we have $B^* = [a^{(1)} \ldots a^{(K)}a^{(1)}\ldots a^{(K-1)}]$.
The average reward of $B^*$ among the optimal sequence is $1$, which is equal to the global average reward.
Now, it is easy to check that playing repeatedly $B^*$ yields an average reward of $K/(2K-1)$.
On the other hand, the lower bound given by \Cref{prop:approx_cyclic} is $1 - K/d =(K-1)/(2K-1)$, which matches the average reward as $K \rightarrow +\infty$.
\end{example}

\section{PROPOSED APPROACH}
\label{sec:method}

We now introduce our approach to solve LSD bandits.
It is based upon the \texttt{CombUCB1} algorithm introduced in \citep{gai2012combinatorial} to solve Combinatorial Semi-Bandits.
Our adaptation is designed to cancel the interference created by the blocks previously played, and enjoys an approximation-estimation tradeoff which can be solved by an appropriate choice of the block size.
In order to keep the exposition concise, we only focus here on the case where the $\mu_i$ are constant on $\mathbb{Z}^-$.
Results for the general case can be found in the Supplementary Material.


\paragraph{Approximation.}
\label{subsec:approx_isi}
A general idea that we can retain from \Cref{subsec:hard}, and from \Cref{prop:approx_cyclic} in particular, is that approximating the optimal sequence using a series of smaller blocks of length $d < T$ is reasonable.
However, the block $B^*$ exhibited during the proof of \Cref{prop:approx_cyclic} is seemingly impossible to estimate, as it requires the computation of the optimal sequence first.
Another key challenge in LSD bandits is to handle the impact of the state: the same block played in different states may have different outcomes, making it difficult to identify ``good blocks''.
One natural way to tackle this problem is to play every block twice: the rewards obtained during the second play are then fully representative of the block value, if repeated as a cycle.
Let $\btau_B$ be the state reached by the system after a play of block $B$ from initial state $\mathbf{1}$.
We can then introduce
\begin{equation}\label{eq:pbm_double}
\Bdouble = \argmax_{|B| = d} ~ r(B \mid \btau_B)~.
\end{equation}
Playing cyclically $\Bdouble$ almost benefits from the same approximation guarantee as playing cyclically $B^*$ (\Cref{prop:approx_all}).
However, solving \eqref{eq:pbm_double} is nontrivial, as $B$ influences both the reward-generating sequence and the initial state (it is essentially equivalent to compute the optimal cycle, which has been proven NP-hard in \Cref{prop:hardness_cyclic}).
Moreover, using an entire block simply to control the initial state of the system seems quite excessive, especially when $d$ which has to grow to control the approximation error.

Instead, a cheaper way to \emph{calibrate} the arms, i.e., to pull the arms before playing the real block such that the arms are in a controlled state independent of the past, consists of playing $B_\sigma = [a^{(\sigma(1))} \ldots a^{(\sigma(K))}]$ for any permutation $\sigma \in \mathfrak{S}_K$.
Indeed, even if no reward can be exploited from $B_\sigma$, it only takes $K$ pulls to calibrate the system, while pre-playing the block is a $d > K$ long calibration.
Let $\btau_\sigma$ be the state reached by the system after a play of block $B_\sigma$ (note that it only depends on $\sigma$), we set
\begin{equation}\label{eq:pbm_calib}
B_\sigma^* = \argmax_{|B| = d} ~ r(B \mid \btau_\sigma)~.
\end{equation}
A natural idea then consists in playing cyclically the block $[B_\sigma, B^*_\sigma]$ of size $K + d$, and approximation results (independent from $\sigma$) can be derived for this approach, see \Cref{prop:approx_all}.
Still, this strategy is not satisfactory for three reasons: (1) the approximation guarantee is worse than for double plays; (2) playing $B_\sigma$ might still be highly inefficient: all arms are calibrated, while only those present in $B^*_\sigma$ would require calibration; (3) in domains like song recommendation, this would amount to regularly play a representative song of each genre.

To remedy this issue, we introduce the following modified version of the block expected reward, which does not take into account the reward obtained if the arm is played for the first time in the block.
Formally, for any block $B = [a_1 \ldots a_d]$ we define
\[
\rew(B) = \sum_{t=1}^d \mu_{a_t}(\tau_{a_t}(t))\mathbb{I}\{\exists t_0 < t\colon a_{t_0} = a_t \}
\]
where $t$ is indexed with respect to $B$.
Note that $\rew(B)$ is independent from the initial state, so that
\begin{equation}\label{eq:pbm_first_pulls}
\tB = \argmax_{|B| = d} ~ \rew(B)
\end{equation}
is well defined.
The rationale behind $\rew$ is the following: first pulls do not provide reliable rewards because they are influenced by the past. Therefore, we use them as a calibration for the future pulls.
So, only the arms which are used are calibrated.
Observe also that calibration and exploitation phases might be intertwined, which was impossible with $B_\sigma$.
But most importantly, the loss incurred by maximizing $\rew$ and not $r$ is controllable, as we have for any block $B$ and initial state $\btauinit$: $r(B \mid \btauinit) - K \le \rew(B) \le r(B \mid \btauinit)$.
We can now state our complete approximation result.

\begin{proposition}\label{prop:approx_all}
Let $(\mu_i)_{i=1}^K$ be a LSD bandit with $K$ arms and constant expected rewards on $\mathbb{Z}^-$.
Let $T \ge 0$ be the horizon, and $d \ge 0$ that divides $T$.
Let $\rdouble(t)$ be the expected rewards obtained at time step $t$ by the policy playing cyclically $\Bdouble$.
We have
\begin{equation}\label{eq:approx_double}
\frac{1}{T} \sum_{t=1}^T \rdouble(t) \ge \left(1 - \frac{d}{T}\right)\left(\frac{1}{T} \sum_{t=1}^T r^*(t) - \frac{K}{d}\right).
\end{equation}
Let $\sigma \in \mathfrak{S}_K$, and assume that $d+K$ divides $T$.
Let $r_\sigma(t)$ be the expected rewards obtained at time step $t$ by the policy playing cyclically $[B_\sigma, B^*_\sigma]$.
We have
\begin{equation}\label{eq:approx_calibration}
\frac{1}{T} \sum_{t=1}^T r_\sigma(t) \ge \frac{d}{d+K} \left(\frac{1}{T} \sum_{t=1}^T r^*(t)\right) - \frac{K}{d+K}.
\end{equation}
Let $\tilde{r}^*(t)$ be the expected rewards obtained at time step $t$ by the policy playing cyclically $\tB$.
We have
\begin{equation}\label{eq:approx_lower}
\frac{1}{T} \sum_{t=1}^T \tilde{r}^*(t) \ge \frac{1}{T} \sum_{t=1}^T r^*(t) - \frac{K}{d}.
\end{equation}
\end{proposition}

The proof of \Cref{prop:approx_all} is similar to that of \Cref{prop:approx_cyclic}, and essentially combine \Cref{lem:key} with the definitions of $B^*_\text{double}$, $B_\sigma$, and $\widetilde{B}^*$.
Based on these results, using $\tilde{r}(B)$ seems by far the best option.
First, it is immediate to check that~\eqref{eq:approx_lower} is tighter than both~\eqref{eq:approx_double} and~\eqref{eq:approx_calibration} as soon as $(1/T)\sum_{t=1}^T r^*(t) \ge K/d$, which is implicitly assumed for the bounds to be nonvacuous.
Moreover, what cannot be seen from~\eqref{eq:approx_double} is that computing a sequence of blocks with small regret against $\Bdouble$ requires to play each block twice, with no guarantee that the first play will provide any reward, thus dividing \eqref{eq:approx_double} by $2$.
On the contrary, as we see next, we can use the \texttt{CombUCB1} algorithm of \citet{gai2012combinatorial} to estimate $\tB$ with tight regret bounds.


\paragraph{Estimation.}
\label{subsec:algo}
A Combinatorial Semi-Bandit (CSB) ---see, e.g., \citep{audibert2014regret}---is an online learning problem where, at each time step, the learner has to select a subset of $N$ actions in a universe of $L > N$ base actions, under some combinatorial constraints.
The individual rewards for each selected action are then revealed, and the learner receives their sum as total reward\footnote{In \citet{chen2013combinatorial} authors consider more general rewards, but the linear case is sufficient for our application}.
The regret is measured against the best subset of $N$ base actions satisfying the constraints.
Note for instance that a standard $K$-armed bandit is a particular instance of CSB, with $N=1$, $L=K$, and no constraints.
We now show that computing a series of blocks with small regret against $\tB$ with respect to $\tilde{r}$ can be reduced to a CSB problem.

In our case $N=d$ (the blocks contain $d$ actions).
The universe of base actions is however slightly more intricate to determine.
While the analysis of CSB uses the fact that a block can only be filled with a subset of base actions, in our setting the same arm might be used several times in the same block.
In addition to that, in CSB base actions have i.i.d.\ rewards, while the rewards of our arms also depend on the state.
Our solution is to consider a universe of $Kd^2$ base actions: namely a base action is indexed by an arm $i \in \{1,\ldots,K\}$, a state $\tau \in \{1,\ldots,d\}$, and a position $t \in \{1,\ldots,d\}$ in the block.
The state coordinate ensures the i.i.d.\ nature of the rewards, while the time coordinate allows to remove the arm multiplicities, therefore making the map from a block to its representation one-to-one.
This is needed to extract valid sequences from solutions to \eqref{eq:LP}, where the combinatorial constraints (ensuring that only subsets deriving from a sequence of pulls can be selected), are made explicit.
Note that the structure of the problem plays a critical role here, as we derive an action space of size $L = Kd^2$, as opposed to $K^d$ in the absence of any structure.
From a pure estimation point of view, a space of size $Kd$ is enough, as the rewards are independent of the position in the sequence.

\texttt{CombUCB1} is an algorithm introduced by \citet{gai2012combinatorial} to solve CSBs.
Its analysis was later refined by \citet{chen2013combinatorial}.
\citet[Theorem 6]{kveton2015tight} prove a regret bound of order $\mathcal{O}(\sqrt{NL n \log n})$, where $n$ is the number of rounds.
This bound can thus be applied to our setting in a black box fashion, with $n=T/d$, $L=d$, and $N=Kd^2$.
The resulting algorithm, which we call \texttt{ISI-CombUCB1} for Initial States Independent \texttt{CombUCB1}, produces a sequence of blocks with small regret against $\tB$ with respect to $\tilde{r}$.
Combining this result with \Cref{prop:approx_all}, we can bound the regret of \texttt{ISI-CombUCB1} with respect to OPT as follows.
\begin{theorem}\label{thm:regret}
Let $(\mu_i)_{i=1}^K$ be an LSD bandit with $K$ arms and constant expected rewards on $\mathbb{Z}^-$.
Let $T \ge 0$ be the horizon, and choose $d \ge 0$ that divides $T$.
Then \texttt{ISI-CombUCB1}, run with block size $d$ and exploration parameter $\alpha=1.5$, has regret bounded by
\[
R_T \le \frac{KT}{d} + 47d\sqrt{KT\log\frac{T}{d}} + \left(\frac{\pi^2}{3} + 1\right)Kd^3~.
\]
Choosing $d \propto T^{1/4}$, we obtain $R_T = \tilde{\mathcal{O}}(KT^{3/4})$, where $\tilde{\mathcal{O}}$ is neglecting logarithmic factors.
\end{theorem}

Note that the second claim of \Cref{thm:regret} requires a horizon-dependent tuning.
One can use the doubling trick to make the algorithm anytime without harming the bound.
Regarding the optimality of \Cref{thm:regret}, we recall that our approximation result in $\mathcal{O}(KT/d)$ is tight, see e.g., \Cref{ex:tight}.
As for the estimation part, \citet{kveton2015tight} proved a lower bound for \texttt{CombUCB1} that matches the upper bound up to a factor of $\sqrt{\log n}$, where $n$ is the horizon.
Instantiating this lower bound to our case, we obtain an overall lower bound of order $\Omega(KT/d + d\sqrt{KT})$.
\Cref{thm:regret} is thus tight up to a polylogarithmic factor of $\sqrt{\log (T/d)}$.

We now give more details about the implementation of \texttt{ISI-CombUCB1}, which specializes \texttt{CombUCB1} to the minimization of $\tilde{r}$.
After an initialization step ensuring that each base action is played at least once, \texttt{CombUCB1} maintains upper confidence bounds (UCBs) for each base action.
At each round, the algorithm plays the admissible subset of actions with the biggest sum of UCBs, and updates the UCBs according to the individual rewards obtained.
In \texttt{ISI-CombUCB1}, the optimization problem to select the best subset amounts to solve \eqref{eq:pbm_first_pulls}, but using the UCBs instead of the true expected rewards.
As pointed out during the construction of the universe, in our case some actions share the same rewards, and it is actually enough to maintain only $Kd$ UCBs (one for each arm and each delay).
Although \Cref{thm:regret} applies to this version of \texttt{ISI-CombUCB1} as well, the need of taking multiplicities into account makes $Kd^2$ appear in the bound.
This memory-efficient version, which is preferred in practice, is summarized in  \Cref{alg:isi-combucb}.
Note that the initialization is such that pulling an uninitialized arm is always better than pulling any combination of initialized arms.
As a consequence, the latter lasts at most $Kd$ rounds, like in standard \texttt{CombUCB1}.

The key step in \Cref{alg:isi-combucb} consists in computing the block $\widehat{B}_t$ maximizing the current reward estimates.
We now formally describe this optimization problem.
Let $F \in \{0, 1\}^{K \times d}$, such that $F[i, t] = 1$ if and only if arm $i$ is played for the first time in the block at time step $t$.
Let $Y \in \{0, 1\}^{K \times d \times d}$, such that $Y[i, j, t] = 1$ if and only if arm $i$ is played with delay $j$ at time step~$t$.
Let $Z = (F, Y)$ be the $Kd^2$-sized representation we introduced earlier.
Note that the column associated to $t=1$ in $Y$ is filled with zeros, as a pull here is by definition a first pull, and thus encoded in $F$.
This notation of size $Kd(d+1)$ is however more convenient for coherence.
The constraints for $Z$ needed to describe a valid sequence of arm pulls are: Action Consistency (AC), i.e., at each time step, one and only one action is played, Unique First Pull (UFP), i.e., there is only one first pull per arm, First Pulls First (FPF), i.e., first pulls must precede any other pull of the same arm, and Time Consistency (TC), i.e., an arm can be pulled with delay $j$ at time step $t$ only if it was pulled at time step $t-j$, and not pulled since.
These constraints write (for index limits that make sense)
\begin{align}
\forall t & \qquad \sum_i F[i, t] + \sum_{i, j} Y[i, j, t] = 1 \tag{AC}\\
\forall i & \qquad \sum_t F[i, t] \le 1 \tag{UFP}\\
\forall i, t & \qquad \sum_j Y[i, j, t] \le \sum_{s=1}^{t-1} F[i, s] \tag{FPF}\\
\forall i, j, t & \qquad Y[i, j, t] \le F[i, t-j] + \sum_l Y[i, l, t-j]\nonumber\\
&\hspace{2.3cm} - \sum_s Y[i, j-s, t-s] \tag{TC}
\end{align}
The objective function is the sum of the current UCBs for the second actions present in the block, and writes $\sum_{i, j, s} Y[i, j, s]\,U_t(i, j)$.
Noticing that all the relations are linear, we can derive $c_t \in \mathbb{R}^{Kd^2}$, $G \in \mathbb{R}^{Kd^2 + K \times Kd^2}$, $h \in \mathbb{R}^{Kd^2 + K}$, $A \in \mathbb{R}^{d \times Kd^2}$, and $b \in \mathbb{R}^d$, such that the optimization problem writes as

\begin{equation}\label{eq:LP}
\max_{z \in \{0, 1\}^{Kd^2}} ~~ c_t^\top z \qquad \text{s.t.} \quad \left\{ \begin{matrix}Gz \preceq h\\
Az = b
\end{matrix}\right.
\end{equation}
where $z$ is a vector version of $Z$.
Note that \eqref{eq:LP} is an integer linear program, which is NP-hard to solve in general.
This is expected, as our approach enjoys a sublinear linear regret against OPT (which is NP-hard to compute), and is therefore bound to be intractable.
In \citet{simchi2021dynamic}, a Fully Polynomial-Time Approximation Scheme (FPTAS) is used to address a similar problem, see Lemma~6 therein.
However, we highlight that, although similar at first sight, our two ILPs are fundamentally different.
While the authors try to select the best $K$ arms at a fixed time step, we aim at selecting an optimal block, that takes into \mbox{account} the evolution of the rewards among time.
This time constraint is specific to our problem, and prevents from using standard FPTASs.
Instead, we \mbox{propose} a heuristics based on a branch-and-bound-like approach \citep{land2010automatic,clausen1999branch}, where we use a LP relaxation to estimate the value of the objective.
This amounts to testing every admissible (discrete) first action, and then keeping the one maximizing the relaxed objective (in which all subsequent actions are relaxed in $[0,1]$).
The same is repeated to choose the following $d-1$ discrete actions.
Overall, we solve $d$ Linear Programs, of sizes $Kd, 2Kd, \ldots, Kd^2$.
Given a horizon $T$, and choosing $d$ as in \Cref{thm:regret}, the total running time is $\mathcal{O}(K^{5/2}T^{9/4})$.
For more details about the heuristic, the reader is referred to \Cref{sec:heuristic}.
Note that to ease readability we restricted ourselves to positive delays only in the core text.
The complete resolution when negative delays are also considered is detailed in \Cref{sec:negative}.

Although it is generally hard to derive approximation guarantees, we point out that: (1) this approach works well in pratice, delivering the optimal solution in all the cases we tested; (2) as discussed in \citet{kveton2015tight}, if \texttt{ISI-CombUCB1} is run with the approximate solver, \Cref{thm:regret} can be adapted to bound the regret against the best block according to the solver's approximation.
Finally, note that our calibration approach is not affecting the complexity of finding the reward-maximizing block.
Indeed, \texttt{CombUCB1} is also required to solve an integer linear program similar to~\eqref{eq:LP} to find the subset of base actions maximizing the block reward.

\section{EXPERIMENTS}
\label{sec:expe}

In this section, we benchmark \texttt{ISI-CombUCB1} against vanilla \texttt{CombUCB1} and \texttt{OracleGreedy}, which plays $\argmax_i ~ \mu_i\big(\tau_i(t)\big)$ at each time step $t$, and breaks ties randomly.
We measure the algorithms performance in terms of the total cumulative reward, averaged over ten repetitions.
For both \texttt{CombUCB1} and \texttt{ISI-CombUCB1}, the exploration parameter $\alpha$ is set to $1.5$, as in \citep{kveton2015tight}.
Let $d$ be the block size of \texttt{CombUCB1}, which consequently maintains $Kd$ UCBs.
Whenever an arm $a$ is pulled with state $\tau \ge d$ (this might happen as the algorithm is not calibrated), we assume that the algorithm updates the estimate of $\mu_a(d)$ ---it actually becomes an estimate of $\mu_a(\tau \ge d)$.
For \texttt{ISI-CombUCB1} to similarly maintain $Kd$ UCBs, we consider blocks of size $d+1$.
Indeed, recall that rewards are ignored the first time an arm is pulled, so that reaching state $\tau_a = d$ requires (at least) $d+1$ steps.
This ``extra'' first pull does not provide any \mbox{information} (it is only used for calibration purposes), but it is nevertheless considered in the performance.
Note finally that both algorithms use the same heuristic based on branch-and-bound and LP relaxations to compute the reward-maximizing block at each time step.

\begin{algorithm}[!t]
\SetKwInOut{Input}{input}
\SetKwInOut{Init}{init}
\SetKwInOut{Parameter}{Param}
\caption{\texttt{ISI-CombUCB1}}
\Input{~number of arms $K$, block size $d$, horizon $T$}\vspace{0.15cm}
\Init{~$\forall i \le K, j \le d$, \quad $T_1(i, j) = 0$, ~ $\overline{X}_1(i, j) = 0$,\\\vspace{0.1cm}
\hspace{2.72cm}$U_1(i, j) = +\infty$.}

\For{$t$ from $1$ to $T/d$}{\vspace{0.2cm}

    \tcp{Play the best block for $\tilde{r}$ based on the UCBs}\vspace{0.1cm}
    
    Play $\widehat{B}_t = [a_{t,1} \ldots a_{t, d}]$ that maximizes in $(a_s)_{s \le d}$\vspace{-0.3cm}
    \[
    \sum_{s=1}^d U_t\big(a_s, \tau_{a_s}(s)\big)\mathbb{I}\{\exists s_0 < s\colon a_{s_0} = a_s \}
    \]
    \vspace{-0.4cm}
    
    Get rewards $X_t\big(a_{t,1}, \tau_{a_{t,1}}(1)\big) \ldots X_t\big(a_{t,d}, \tau_{a_{t,d}}(d)\big)$
    
    \vspace{0.2cm}
    \tcp{Update the statistics}\vspace{0.1cm}
    \For{$i \le K$ and $j \le d$}{\vspace{0.2cm}
    
    \If{arm $i$ is played with delay $j$ in block $\widehat{B}_t$ (counting the multiplicities)}{
    \vspace{0.2cm}
    
    $T_{t+1}(i, j) = T_t(i, j) + 1$,\\
    
    $\displaystyle\overline{X}_{t+1}(i, j) = \frac{T_t(i, j)\overline{X}_t(i, j) + X_t(i, j)}{T_{t+1}(i, j)}$
    }
    \vspace{0.2cm}
    
    \Else{
    \vspace{0.15cm}
    
    $T_{t+1}(i, j) = T_t(i, j), \quad \overline{X}_{t+1}(i, j) = \overline{X}_t(i, j)$
    }
    \vspace{0.3cm}
    $\displaystyle U_{t+1}(i, j) = \overline{X}_{T_{t+1}(i, j)}(i, j) + \sqrt{\frac{\alpha \log(t+1)}{T_{t+1}(i, j)}}$
    }
    }
\label{alg:isi-combucb}
\end{algorithm}

\begin{figure}[!t]
    \centering
    \includegraphics[width=0.83\columnwidth]{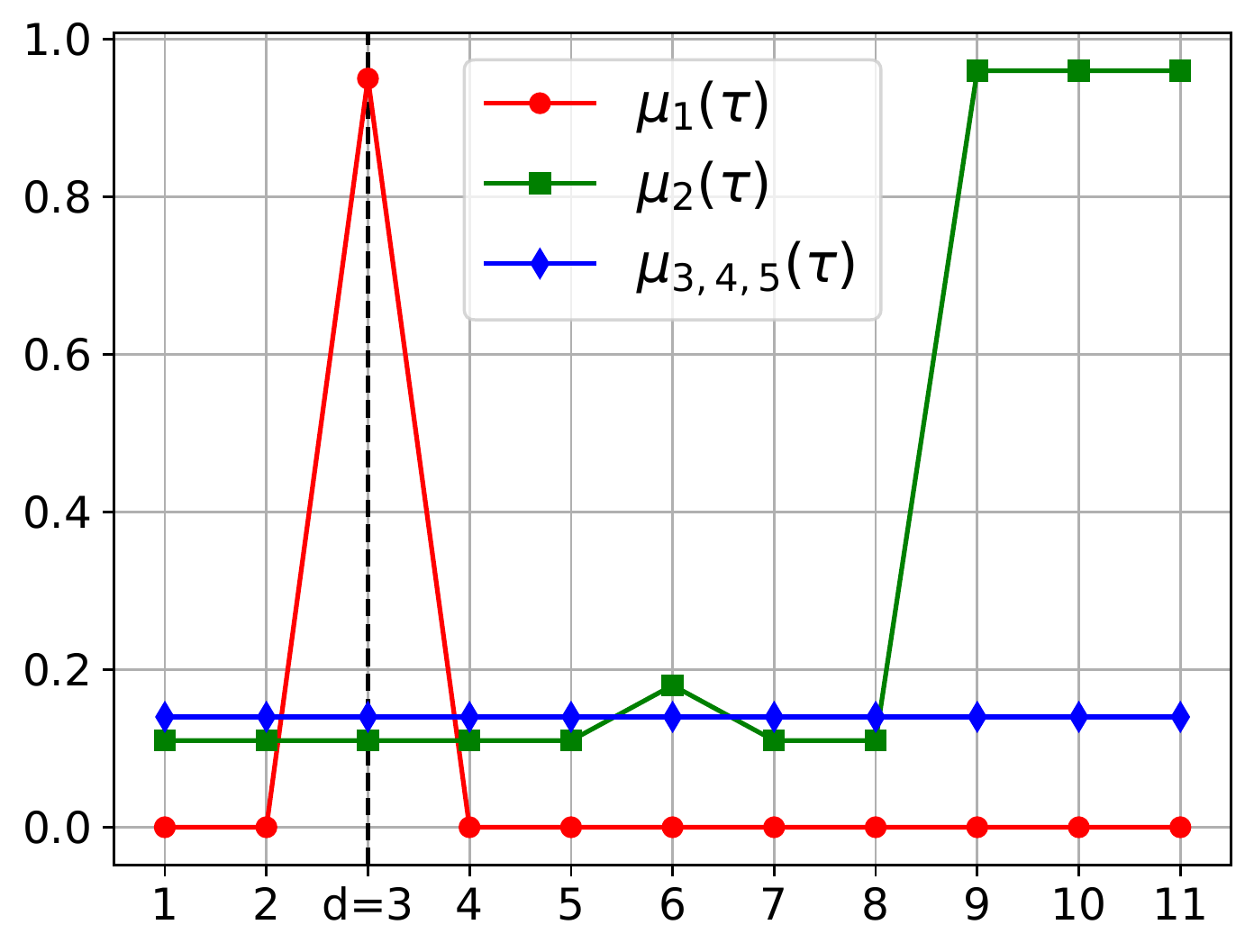}\\[0.3cm]
    ~~\includegraphics[width=0.79\columnwidth]{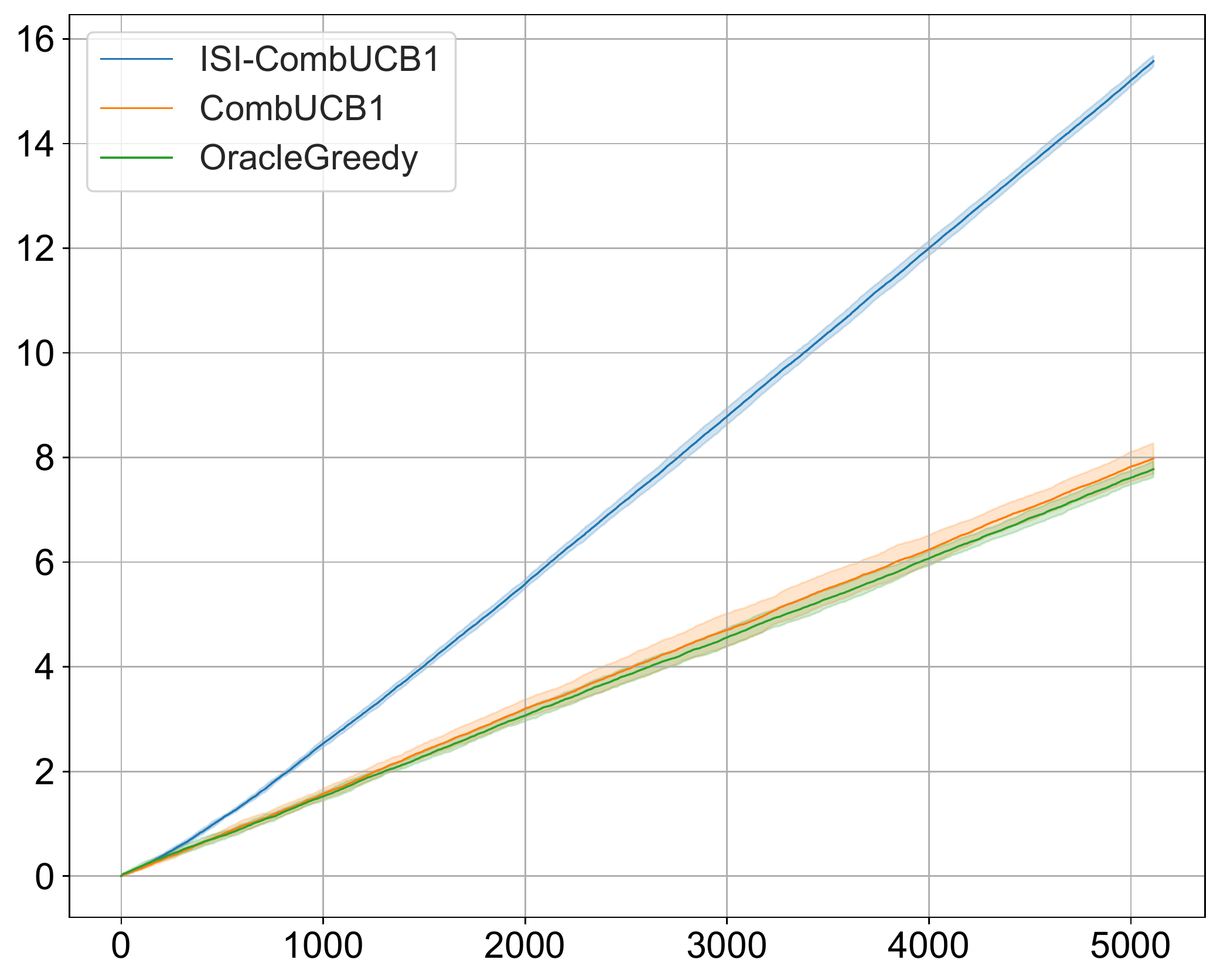}
    \caption{Reward functions with respect to $\tau$ (top) and cumulative rewards in thousands (bottom).}
    \label{fig:exp_figures}
\end{figure}

We consider the 5-armed LSD bandit with Bernoulli rewards and\vspace{-0.25cm}
\[
\mu_1(\tau) = 0.95\,\mathbb{I}\{\tau = 3\}, \quad \mu_2(\tau) = \begin{cases}0.16 &\text{if } \tau = 6\\0.96 &\text{if } \tau \ge 9\\0.14 &\text{otherwise}\end{cases}\,,
\]
and $\mu_{3,4,5}(\tau) = 0.15$ for all $\tau$, see \Cref{fig:exp_figures} (top).
The empirical results for $d=3$ are reported in \Cref{fig:exp_figures} (bottom).
\texttt{ISI-CombUCB1} converges towards the block {$[a^{(1)}, a^{(i)}, a^{(i)}, a^{(1)}]$}, where $i$ is any arm in $\{3, 4, 5\}$.
Here, the first pull of arm $a^{(1)}$ is needed to calibrate the state of this action, and to get an expected reward of $\mu_1(3)$ in the last pull of the block.
On the first 8 time steps, \texttt{OracleGreedy} plays any constant arm, except at $t=3$ and $t=6$, where it plays $a^{(1)}$.
At $t=9$, it faces a choice between $\mu_1(3)$ and $\mu_2(9)$ and thus goes for the latter.
Then, it will never be able to get $\mu_1(3)$ again, and will only get $\mu_2(9)$ every nine steps with any combination of constant actions in between.
As for \texttt{CombUCB1}, it is unable to see the spike at $\mu_1(3)$.
Precisely, whenever the algorithm plays arm $a^{(1)}$ with state $\tau > 3$, the expected reward is small, but still contributes to the UCB of $\mu_1(3)$.
This decreases the UCB for the arm, and the algorithm pulls it less and less frequently.
More experiments testing the calibration sequence approach are provided in the Appendices.

\par{\bf Acknowledgments.}{
This work has been partially supported by the academic grant \textit{Innovative Technologies, Interventions and Novel Approaches in the Field of Neuro-Wellness} from \href{https://www.joyventures.com}{Joy Ventures} and by the EU Horizon 2020 ICT-48 research and innovation action under grant agreement 951847, project ELISE (European Learning and Intelligent Systems Excellence).}

\bibliography{ref}

\clearpage
\appendix
\onecolumn
\section{TECHNICAL PROOFS}
\label{apx:proofs}

In this section, we gather all the technical proofs omitted in the main body of the paper.


\subsection{Proof of Proposition \ref{prop:hardness}}
\label{apx:hardness}

First note that focusing on the case where the $\mu_i$ are constant on $\mathbb{Z}^-$ is enough to prove NP-hardness.
Within this setting, we actually prove a slightly more general result than \Cref{prop:hardness}, namely that for any $\btauinit \in \mathbb{N}^K$, computing
\begin{equation}\label{eq:hard_pbm}
B^* = \argmax_{|B| = d} ~ r(B \mid \btauinit)
\end{equation}
is NP-hard as soon as $d$ is greater than a certain value $d_0$ precised later on.
In particular, when $\btauinit = \bm{1}$, and $d = T$, solving \eqref{eq:hard_pbm} amounts to find the best policy.
For some $T$ large enough, we get $d \ge d_0$, and we obtain that computing the optimal policy is NP-hard (\Cref{prop:hardness}).
But more interestingly, this result also highlights that the inner optimization problem of a \texttt{CombUCB1} approach is itself NP-hard.
Note that the same result holds for the inner optimization problem of \texttt{ISI-CombUCB1}, as the hardness can be similarly proved with $\tilde{r}$ instead of $r$ in \eqref{eq:hard_pbm}.

Our reduction is largely adapted from the one in \citep{basu2019blocking}.
Indeed, our framework encapsulates the setting described by Equations \eqref{eq:setting_pw}, which is the one used by the authors to show the hardness of their problem.
The main differences are: 1) we need to take into account the initial state $\btauinit$, which can be handled by considering a larger block size $d$, and 2) in \citep{basu2019blocking}, any empty slot in the scheduling is automatically filled with a pull of the $0$ arm (it is the only arm which can be played), while we have to invoke a finer argument, namely that at any empty slot, playing the $0$ arm is the best possible action, since all actions yield a null expected reward, but playing the $0$ arm allows every other arms to recharge.

As in \citep{basu2019blocking}, we thus consider the Pinwheel Scheduling Problem (PSP), see \citep{holte1989pinwheel}, which can be stated as follows.
Given a set of actions $\{1, \dots, K\}$, associated to delays $(d_i)_{i=1}^K \in \mathbb{N}^K$, the PSP consists in deciding whether it is possible or not to design an action scheduling, i.e., a mapping from $\mathbb{N}$ to $\{1, \ldots, K\}$, such that each arm $i$ is played at least once in any sequence of $d_i$ consecutive pulls.
A PSP instance with such a scheduling is called a YES instance.
Otherwise, it is referred to as a NO instance.
Furthermore, a PSP instance is said to be dense if $\sum_{i=1}^K 1/d_i = 1$.
In particular, this implies for YES instances that each arm $i$ is played exactly every $d_i$ pulls.
It has been shown in \citep{bar2002minimizing} that PSP on dense instances is NP-complete.
In the next paragraph, we show that PSP on dense instances reduces to particular instances of our problem, described by Equations \eqref{eq:setting_pw}, proving the hardness of our problem.

Given a dense PSP instance, we construct an instance of our problem as follows.
For every arm $i \leq K$, we set:
\begin{equation}\label{eq:setting_pw}
\mu_i(\tau) = \mathbb{I}\{ \tau \geq d_i \}. 
\end{equation}
We set an additional arm $K+1$, such that $\mu_{K+1}(\tau) = 0$ for all $\tau$.
Given an initial state $\btauinit$, we want to find the best block of actions $B^* = \argmax_{\lvert B \rvert = d} ~ r(B \mid \btauinit)$.
Then, we have two possibilities:
\begin{itemize}
\item if the PSP is a YES instance, the PSP schedule gives a $1$ reward at each time step.
Taking into account the fact that the initial state might not exactly suit the schedule, we obtain that $r(B^* \mid \btauinit, \text{YES}) \ge d - K$.
\item if the PSP is a NO instance, we can use the same argument as in \citep{basu2019blocking} and prove that $r(B^* \mid \btauinit, \text{NO}) \le d - \left\lfloor \frac{d}{\prod_{i=1}^K d_i} \right\rfloor$.
\end{itemize}

Hence, for $d \ge d_0 \coloneqq (K+1) \prod_{i=1}^K d_i$, if we were able to compute $B^*$, and therefore $r(B^* \mid \btauinit)$, we could discriminate YES from NO instances, depending on whether $r(B^* \mid \btauinit) \ge d - K$ or not.
Solving PSP on dense instances thus reduces to solving \eqref{eq:hard_pbm} for the particular choice of reward functions explicited in \eqref{eq:setting_pw}, therefore proving the hardness of the latter problem.\qed


\subsection{Proof of Lemma \ref{lem:cyclic}}
\label{apx:cyclic}

Assume without loss of generality that $a^{(1)}$ is the first action played.
If there is no switch of actions, then $a^{(1)}$ is played indefinitely, which means that the sequence is cyclic.
If there is only one (respectively two) switch(es) of actions, then $a^{(2)}$ (respectively $a^{(1)}$) is played indefinitely after some time step, from which the sequence is cyclic.
If there are three (or more) switches of actions, then the sequence witnesses twice the switch of actions $[a^{(1)}, a^{(2)}]$, and thus visits twice the state $(\tau_1, \tau_2) = (1, -1)$.
Since the policy is deterministic, the same cycle will be repeated.\qed

Note however that this property does not resist to the number of arms, as we can exhibit optimal policies with $3$ arms that never visit the same state twice.
Let $\A = \{a^{(1)}, a^{(2)}, a^{(3)}\}$, and consider the sequence of actions built as follows.

For $m = 1, 2, \ldots$ play\vspace{-0.3cm}
\begin{itemize}
    \item $a^{(1)}$ $m$ times consecutively,
    \item $a^{(2)}$ $m$ times consecutively,
    \item $a^{(3)}$ $m$ times consecutively.
\end{itemize}

Assuming for simplicity that the states $\tau_i$ are initialized to $0$ and not to $1$ (this only impacts the states in the first block), it can be checked that, in block $m$, the state $(\tau_1, \tau_2, \tau_3)$ is given by
\[
\begin{cases}
(-t,~m-1+t,~t) & \text{after the $t^\text{th}$ pull of } a^{(1)}\\
(t,~-t,~m+t) & \text{after the $t^\text{th}$ pull of } a^{(2)}\\
(m+t,~t,~-t) & \text{after the $t^\text{th}$ pull of } a^{(3)}
\end{cases}
\]
From the above formula, it is immediate to check that no state is visited twice by the sequence.
Now it is easy to choose $\mu_1,\mu_2,\mu_3$ such that this policy is optimal.
For example, $\mu_i \equiv 1$ for $i=1,2,3$.


\subsection{Proof of Proposition \ref{prop:approx_cyclic}}
\label{apx:approx_cyclic}

Recall that \Cref{prop:approx_cyclic} is proven in the main body of the paper for the special case where the $\mu_i$ are constant on $\mathbb{Z}^-$.
The proof in the general case is very similar, but uses the second claim of \Cref{lem:key}, and more specifically the way $B'$ is constructed from $B$.

Similarly to the constant case, we know that there exists a block $B^*$ of length $d$ with average reward (at least) greater than the average reward of the complete optimal sequence.
Recalling the notation from the proof in the constant case, we have $B^* = [a^*_{t_0} \ldots a^*_{t_0+d-1}]$ and
\[
\frac{r(B^* \mid \btau^*(t_0))}{d} \ge \frac{1}{T} \sum_{t=1}^T r^*_t\,.
\]
Previously, we could simply repeat $B^*$, which is not possible anymore due to the possible effects showed in \eqref{eq:bad}.
Let $B'$, derived from $B$ as in the proof of \Cref{lem:key}.
Namely, if $\aone$ and $\atwo$ are the first two different actions played in $B$, we have $B' = [a^*_{t_0}\, \atwo\, a^*_{t_0 + 1} \ldots a^*_{t_0 + d - 1}]$.
Let $\btau_{B'}$ be the state reached after a play of $B'$ from initial state $\bm{1}$.
Using \Cref{lem:key}, the expected rewards $r_\pi(t)$ obtained by policy $\pi$ which plays cyclically $B'$ satisfy
\begin{align*}
\frac{1}{T}\sum_{t=1}^T r_\pi(t) &= \frac{r(B' \mid \bm{1})}{T} + \frac{T-d}{T}~\frac{r(B' \mid \btau_{B'})}{d}\\
&= \frac{d}{T}~\frac{r(B' \mid \bm{1})}{d} + \left(1 - \frac{d}{T}\right)\frac{r(B' \mid \btau_{B'})}{d}\\
&\ge \frac{d}{T}~\frac{r(B^* \mid \btau^*(t_0)) - (K + 2)}{d} + \left(1 - \frac{d}{T}\right)\frac{r(B^* \mid \btau^*(t_0)) - (K+2)}{d}\\
&= \frac{r(B^* \mid \btau^*(t_0)) - (K + 2)}{d}\\
&\ge \frac{1}{T} \sum_{t=1}^T r^*_t - \frac{K+2}{d}\,.
\end{align*}
\qed


\subsection{Proof of Proposition \ref{prop:hardness_cyclic}}
\label{apx:hardness_cyclic}

Similarly to \Cref{prop:hardness}, we only need to focus on the case where the reward functions are constant on $\mathbb{Z}^-$, and actually prove here a stronger result than \Cref{prop:hardness_cyclic}.
Namely, we show that for some $d$, computing
\begin{equation}\label{eq:hard_pbm_cyclic}
B^* = \argmax_{|B| = d} ~ r(B \mid \btau_B)
\end{equation}
where $\btau_B$ is the state reached by the system after a play of block $B$, is NP-hard, even when the reward functions can be totally ordered.
The optimal cyclic policy (with cycle length $d$) is obtained by repeating indefinitely the solution to \eqref{eq:hard_pbm_cyclic}, and the NP-hardness of the latter problem then yields \Cref{prop:hardness_cyclic}.

This proof is also based on a reduction of the Pinwheel Scheduling Problem (PSP), see \Cref{apx:hardness}.
Given a PSP dense instance $(d_i)_{i=1}^K$, we construct an instance of our problem as follows.
For every arm $i=1,\ldots,K$, we set:
\begin{equation}\label{eq:setting_pw_2}
\mu_i(\tau) = \sqrt{\tau/d_i}\,.
\end{equation}
Note that setting \eqref{eq:setting_pw_2} is fundamentally different from setting \eqref{eq:setting_pw} as we have a total ordering of the arms, i.e., for all $\tau$ we have $\mu_1(\tau) \ge \mu_2(\tau) \ge \ldots \ge \mu_K(\tau)$, with the convention $d_1 \le d_2 \le \ldots \le d_K$.
We also highlight that we do not introduce an additional null arm here, as opposed to \eqref{eq:setting_pw_2}.
Finally, although the $\mu_i$ are unbounded for the moment, thus breaking \Cref{def:LSD}, we see at the end of the proof how to consider bounded $\mu_i$ without altering the subsequent analysis.

We now show that solving \eqref{eq:hard_pbm_cyclic} with reward functions \eqref{eq:setting_pw_2} allows us to determine if the PSP instance is a YES or a NO instance.
Let $n_i$ be the number of times action $i=1,\ldots,K$ is played in the block, and $\tau_{ij}$, for $j = 1,\ldots,n_i$ be the different states in which arm $i$ is pulled.
Problem \eqref{eq:hard_pbm_cyclic} is equivalent to: 
\begin{equation}\label{eq:psp_problem}
\max_{(n_i)_{i=1}^K} \max_{(\tau_{ij})_{i=1, j=1}^{K, n_i}} \quad \sum_{i=1}^K \sum_{j=1}^{n_i} \sqrt{\tau_{ij} / d_i} \text{\qquad subject to \qquad} \begin{cases}\sum_{i=1}^K n_i = d\\[0.2cm]\sum_{j=1}^{n_i} \tau_{ij} = d \quad i = 1,\ldots,K
\end{cases}
\end{equation}
We start by maximizing with respect to the $\bm{\tau} = (\tau_{ij})_{i=1, j=1}^{K,n_i}$.
The Lagrangian writes
\[
\mathcal{L}(\bm{\tau}, \bm{\lambda}) = \sum_{i=1}^K \sum_{j=1}^{n_i} \sqrt{\tau_{ij} / d_i} + \sum_{i=1}^{K} \lambda_i \left(\sum_{j=1}^{n_i}\tau_{ij} - d\right)\,.
\]
The KKT conditions (gradient of the Lagrangian and primal feasibility) write
\begin{align}
\label{eq:kkt1}
\frac{\partial \mathcal{L}(\bm{\tau}, \lambda)}{\partial \tau_{ij}} &= \frac{1}{2 \sqrt{d_i \tau_{ij}}} + \lambda_i = 0 \hspace{-2.5cm}&&i=1,\ldots,K,\, j=1,\ldots,d\\
\label{eq:kkt2}
\sum_{j=1}^{n_i} \tau_{ij} &= d &&i=1,\ldots,K.
\end{align}
Solving \eqref{eq:kkt1} for $\tau_{ij}$ we obtain a quantity independent of $j$.
Then \eqref{eq:kkt2} implies that $\tau_{ij} = d/n_i$ for each $i=1,\ldots,K$.
Replacing $\tau_{ij} = d/n_i$ into \eqref{eq:psp_problem}, we can now maximize with respect to the $n_i$.
The Lagrangian writes
\[
\mathcal{L}(\bm{n}, \lambda) = \sum_{i=1}^K \sqrt{dn_i / d_i} + \lambda\left(\sum_{i=1}^K n_i - d\right)
\]
and the KKT conditions (gradient of the Lagrangian and primal feasibility) are
\begin{align}
\label{eq:kkt3}
\frac{\partial \mathcal{L}(\bm{n}, \lambda)}{\partial n_i} &= \sqrt{\frac{d}{4 n_i d_i}} + \lambda = 0 \qquad i = 1,\ldots,K \\
\label{eq:kkt4}
\sum_{i=1}^K n_i &= d
\end{align}
such that replacing \eqref{eq:kkt3} into \eqref{eq:kkt4}, we obtain $n_i = d/d_i$, which implies $\tau_{ij} = d_i$.

Assume now that $d$ can be divided by all the $d_i$, such that $n_i = d/d_i \in \mathbb{N}$.
From the values of $n_i$ and $\tau_{ij}$ obtained, we can see that the optimal block for \eqref{eq:hard_pbm_cyclic} corresponds to a Pinwheel schedule.
It yields an average reward equal of $1$, and is achievable if and only if the Pinwheel instance is a YES instance.
Therefore, if we can solve \eqref{eq:hard_pbm_cyclic}, we can tell if the average reward is equal to $1$ or smaller, and thus decide whether the instance is a YES or NO instance.
We have reduced PSP on dense instances to \eqref{eq:hard_pbm_cyclic}, which is therefore shown to be NP-hard.
To our knowledge, this is the first hardness result for decomposable reward functions of the form $\mu_i(\tau) = \mu_i^0 \cdot \gamma(\tau)$.

Now, let $d_\text{max} = \max_{i=1,\ldots,K} d_i$.
Note that replacing \eqref{eq:setting_pw_2} with the bounded functions
\[
\mu_i(\tau) = \begin{cases}
\sqrt{\tau/d_i}         &\text{if } \tau \le d_\text{max}\\
\sqrt{d_\text{max}/d_i} &\text{otherwise}
\end{cases}
\]
does not change change the optimal schedule (arm $i$ is played every $d_i \le d_\text{max}$ time steps) and the analysis is unchanged.\qed


\subsection{Proof of Lemma \ref{lem:key}}
\label{apx:key}

Recall that \Cref{lem:key} is proven in the main body of the paper for the special case where the $\mu_i$ are constant on $\mathbb{Z}^-$.
We focus here on the general version only.
For any block $B$ of size $d$, we want to find $B'$ of length $d$ such that for any initial states $\btauinit, \btauinit' \in \mathbb{Z}^K$, we have
\begin{equation}\label{eq:recall}
r(B' \mid \btauinit') \ge r(B \mid \btauinit) - (K + 2)\,.
\end{equation}

First, we analyze what happens if we keep the same $B$, instead of $B'$, in the left-hand side of \eqref{eq:recall}.
Similarly to the constant case, it is important to note that an action is impacted by a change of initial state only the first time it is played in the block (possibly with several consecutive pulls).
Indeed, when this sequence (which might be of length $1$ if a switch follows the first pull) is interrupted, the arm goes to state $1$, no matter the state it was before, and similar subsequent actions then yield similar states/rewards.
A second point to notice is that only the first action $a_1$ in the block can be pulled with a negative state at the beginning of its sequence of pulls.
Indeed, let $\btau_\text{new} = \bm{\delta}(\btauinit, a_1)$, where $\bm{\delta}$ is a componentwise version of $\delta_a$, and $a_1$ is the first action played in $B$.
It is easy to check that for all $a \ne a_1$, we have $\tau_{\text{new}, a} \ge 1$ (either the action was in a negative state and was set to 1 as $a_1 \ne a$ has been played, or it was in a positive state and incremented by 1).
Combining these two remarks, we know that for any action $a$ different from $\aone$ (we can assume that $a_1 = \aone$ without loss of generality), the loss incurred by a change of initial state due to the pulls of $a$ is bounded by $1$.
Indeed, the expected rewards obtained during the first play (possibly with consecutive pulls) with initial states $\btauinit$ and $\btauinit'$ are respectively
\begin{equation}\label{eq:rewards_collected}
\mu_a(\tau_a),\, \mu_a(-1),\, \mu_a(-2), \ldots \text{\qquad and \qquad} \mu_a(\tau'_a),\, \mu_a(-1),\, \mu_a(-2), \ldots
\end{equation}
where $\tau_a$ and $\tau'_a$ are two generic \textbf{positive} values (thanks to the remark we made earlier) that depend on $\tau_{\text{init}, a}$, $\tau'_{\text{init}, a}$, and the place of the first pull of $a$ in the block.
Making these values explicit is not important here, as the boundedness of $\mu_a$ ensures anyway that the difference of expected rewards obtained is smaller than $1$.

Now, what happens for the first pulls of $a^{(1)}$?
We assume that it is played $n_1$ times consecutively at the beginning of $B$ (again, we might have $n_1=1$).
Assuming that $\tau_{\mathrm{init}, 1}$ is positive and $\tau_{\mathrm{init}, 1}'$ is negative, the expected rewards collected are respectively
\begin{equation}\label{eq:bad}
\mu_1(\tau_{\mathrm{init}, 1}),\, \mu_1(-1)\, \ldots\, \mu_1(-n_1+1) \text{\qquad and \qquad} \mu_1(\tau_{\mathrm{init}, 1}'),\, \mu_1(\tau_{\mathrm{init}, 1}'-1)\, \ldots\, \mu_1(\tau_{\mathrm{init}, 1}'-n_1+1)\,.
\end{equation}
Unlike in the previous case, the difference between these two sequences cannot be bounded by $1$ (it might even be equal to $n_1$ if $\tau_{\mathrm{init}, 1}' \le -n_1$).
This is why the same $B$ cannot be used on both sides of \eqref{eq:recall}.
To break this sequence of pulls, we define $B'$ as follows.
Without loss of generality, let $a^{(2)}$ be the second different action played in $B$ (if $B$ is only composed of pulls of $a^{(1)}$, we can set $a^{(2)}$ to be any action of $\mathcal{A}$ different from $a^{(1)}$).
The block $B'$ is equal to $B$, except that the second pull of $B'$ is necessarily $a^{(2)}$.
We may now face $3$ different cases.
\begin{itemize}
\item $n_1=1$. Note that here $B' = B$. Since $n=1$, the difference between the two sequences of rewards due to the pulls of $a^{(1)}$ is at most $1$. For all other actions, we can use the analysis we developed at the beginning, and the total difference is at most $K$.
\item $n_1=2$. Then, denoting by $n_2$ the number of times $\atwo$ is played consecutively after $\aone$, the expected rewards obtained by $B$ with initial state $\btauinit$ and by $B'$ with initial state $\btauinit'$ are respectively
\[\begin{matrix}
& \red{\mu_1(\tau_{\text{init}, 1})}, & \red{\mu_1(-1)} \text{ or } \red{\mu_1(\tau_{\text{init}, 1} - 1)}, & \red{\mu_2(\tau_2)}, & \mu_2(-1) & \ldots & \mu_2(-n_2 + 1)\\[0.2cm]
\text{and} & \red{\mu_1(\tau'_{\text{init}, 1})}, & \red{\mu_2(\tau'_2)}, & \mu_2(-1), & \mu_2(-2) & \ldots & \red{\mu_2(-n_2)}
\end{matrix}\]
where the \emph{or} comes from the fact that we don't know if $\tau_{\mathrm{init}, 1}$ is positive (then the next reward is $\mu_1(-1)$) or negative (then next reward is $\mu_1(\tau_{\text{init}, 1} - 1)$). Here, $\tau_2$ and $\tau'_2$ are two generic positive numbers, whose values are not important as the difference between the two sequences is contained in the red rewards, and thus bounded by $3$ anyway. For all other $K-2$ actions, we can apply the standard analysis, such that in total the difference cannot exceed $K+1$.
\item $n_1 \ge 3$. Using the same notation as above, we have now respectively for the first $n_1$ pulls
\[\hspace{-0.7cm}\begin{matrix}
& \red{\mu_1(\tau_{\text{init}, 1})}, & \red{\mu_1(-1)} \text{ or } \red{\mu_1(\tau_{\text{init}, 1} - 1)}, & \red{\mu_1(-2)} \text{ or } \red{\mu_1(\tau_{\text{init}, 1} - 2)}, & \ldots & \mu_1(-n_1 + 1) \text{ or } \mu_1(\tau_{\text{init}, 1} - n_1 + 1)\\[0.2cm]
\text{and} & \red{\mu_1(\tau'_{\text{init}, 1})}, & \red{\mu_2(\tau'_2)}, & \red{\mu_1(1)}, & \ldots & \mu_1(-n_1 + 3)
\end{matrix}\]
The red terms incur a loss of at most $3$.
Then, if $\tau_{\text{init}, 1} \ge 1$, the remaining rewards (non red) in the play of $B$ are  $\sum_{j=3}^{n_1-1} \mu_1(-j) \le \sum_{j=1}^{n_1-3} \mu_1(-j)$ as $\mu_1$ is nondecreasing on $\mathbb{Z}^-$, so the rewards obtained by $B'$ are greater.
If $\tau_{\text{init}, 1} < 0$, we have $\sum_{j=3}^{n_1-1} \mu_1(\tau_{\text{init}, 1} - j) \le \sum_{j=1}^{n_1-3} \mu_1(-j)$ for the same reasons.
So overall, the difference is bounded by $3$.

As for the $n_2$ following pulls, recalling that $\tau_2$ and $\tau'_2$ are positive since the preceding pull is $a^{(1)}$, we have
\[\begin{matrix}
&\red{\mu_2(\tau_2)}, & \mu_2(-1), & \ldots & \mu_2(-n_2 + 1)\\[0.2cm]
\text{and} & \red{\mu_2(\tau'_2)}, & \mu_2(-1), & \ldots & \mu_2(-n_2 + 1)
\end{matrix}\]
with a difference bounded by $1$.
For the $K-2$ other actions, we still have the bound of $K-2$, so that in total the difference does not exceed $K+2$.
\end{itemize}
\vspace{-0.4cm}
\qed


\subsection{Proof of Proposition \ref{prop:approx_all}}
\label{apx:approx_all}

Recall the notation $B^*_\text{double}$, and the additional notation $B^*$ and $\btau^*(t_0)$ introduced in the proof of \Cref{prop:approx_cyclic} in the main body of the paper.
We have
\begin{align}
\nonumber
\frac{1}{T}\sum_{t=1}^T \rdouble(t) &= \frac{r(\Bdouble \mid \bm{1})}{T} + \frac{T-d}{T}~\frac{r(\Bdouble \mid \btau_{\Bdouble})}{d}\\
\nonumber
&\ge \left(1 - \frac{d}{T}\right)\frac{r(\Bdouble \mid \btau_{\Bdouble})}{d}\\
\label{eq:A1}
&\ge \left(1 - \frac{d}{T}\right)\frac{r(B^* \mid \btau_{B^*})}{d}\\
\label{eq:A2}
&\ge \left(1 - \frac{d}{T}\right)\frac{r\big(B^* \mid \btau^*(t_0)\big) - K}{d}\\
\label{eq:A3}
&\ge \left(1 - \frac{d}{T}\right)\left(\frac{1}{T}\sum_{t=1}^T r^*(t) - \frac{K}{d}\right)\,,
\end{align}
where \eqref{eq:A1} holds because $\Bdouble$ is a maximizer of $r(B \mid \btau_B)$, \eqref{eq:A2} holds due to \Cref{lem:key}, and \eqref{eq:A3} is a direct consequence of the definition of $B^*$.
Using similar arguments, we also have
\[
\frac{1}{T}\sum_{t=1}^T r_\sigma(t) \ge \frac{r(B^*_\sigma \mid \btau_{B_\sigma})}{d + K} \ge \frac{d}{d+K} ~ \frac{r(B^* \mid \btau_{B_\sigma})}{d} \ge \frac{d}{d+K} ~ \frac{r\big(B^* \mid \btau^*(t_0)\big) - K}{d} \ge \frac{d}{d+K}\left(\frac{1}{T} \sum_{t=1}^T r^*(t)\right) - \frac{K}{d + K}\,,
\]

and
\[
\frac{1}{T}\sum_{t=1}^T \rew^*(t) \ge \frac{\rew(\tB)}{d} \ge \frac{\rew(B^*)}{d} \ge \frac{r\big(B^* \mid \btau^*(t_0)\big) - K}{d} \ge \frac{1}{T} \sum_{t=1}^T r^*(t) - \frac{K}{d}\,.
\]
\qed

\section{GENERAL STUDY WHEN \texorpdfstring{$\mu_a$}{mua} IS NOT CONSTANT ON \texorpdfstring{$\mathbb{Z}^-$}{Z}}

In this section, we detail the results in the general case where the $\mu_i$ are not constant on $\mathbb{Z}^-$, that were omitted in \Cref{sec:method} for simplicity.
If the definitions of $\Bdouble$ and $B^*_\sigma$ remain unchanged, we need to adapt the definition of $\tB$.
Indeed, the goal of $\rew$ is to use first pulls as a calibration step, such that subsequent pulls are in a controlled state.
However, assume that the system is in a state with unknown $\tau_a \le 0$, and that \texttt{ISI-CombUCB1} returns a block starting with several pulls of $a$.
None of the rewards obtained by the sequence can be used to update our estimates, as they are obtained in states $\tau_a, \tau_a -1, \tau_a -2, \ldots$ which are all unknown.
This problem is avoided with constant $\mu_a$ on $\mathbb{Z}^-$, as the rewards would be obtained in states $\tau_a, -1, -1, \ldots$ such that the first pull indeed plays its calibration role.
Therefore, we now define $\tB$ as follows
\begin{equation}\label{eq:new_B}
\tB = \argmax_{|B| = d,~a_1 \ne a_2} ~ \rew(B)\,.
\end{equation}
The fact that the first two actions in $\tB$ are now different prevents the issues described above.
On the other hand, note that $B'$ in the second claim of \Cref{lem:key} also possesses two different first actions, such that the extra constraint in \eqref{eq:new_B} is not harmful in terms of approximation.


\subsection{Equivalent of Proposition \ref{prop:approx_all} and Theorem \ref{thm:regret}}

\Cref{prop:new} and \Cref{thm:new} are the analog of \Cref{prop:approx_all} and \Cref{thm:regret} respectively.

\begin{proposition}\label{prop:new}
Let $(\mu_i)_{i=1}^K$ be a LSD bandit with $K$ arms.
Let $T \ge 0$ be the horizon, and $d \ge 0$ that divides $T$.
Let $\rdouble(t)$ be the expected rewards obtained at time step $t$ by the policy playing cyclically $\Bdouble$.
We have
\[
\frac{1}{T} \sum_{t=1}^T \rdouble(t) \ge \left(1 - \frac{d}{T}\right)\left(\frac{1}{T} \sum_{t=1}^T r^*(t) - \frac{K + 2}{d}\right)\,.
\]
Let $\sigma \in \mathfrak{S}_K$, and assume that $d+K$ divides $T$.
Let $r_\sigma(t)$ be the expected rewards obtained at time step $t$ by the policy playing cyclically $[B_\sigma, B^*_\sigma]$.
We have
\[
\frac{1}{T} \sum_{t=1}^T r_\sigma(t) \ge \frac{d}{d+K} \left(\frac{1}{T} \sum_{t=1}^T r^*(t)\right) - \frac{K + 2}{d+K}\,.
\]
Let $\tilde{r}^*(t)$ be the expected rewards obtained at time step $t$ by the policy playing cyclically $\tB$.
We have
\[
\frac{1}{T} \sum_{t=1}^T \tilde{r}^*(t) \ge \frac{1}{T} \sum_{t=1}^T r^*(t) - \frac{K + 2}{d}\,.
\]

\end{proposition}

\begin{proof}
The proof is similar to that of \Cref{prop:approx_all}, see \Cref{apx:approx_all}.
The only difference is that we cannot use $r(B^* \mid \btau_{B^*}) \ge r\big(B^* \mid \btau^*(t_0)\big)- K$, as we did during the proof of the first claim of \Cref{prop:approx_all}.
Instead, we have to involve $(B^*)'$, as defined in \Cref{lem:key}.
Then, we have
\begin{gather*}
r(\Bdouble \mid \btau_{\Bdouble}) \ge r\big((B^*)' \mid \btau_{(B^*)'}\big) \ge r\big(B^* \mid \btau^*(t_0)\big)- (K+2)\,,\\[0.2cm]
r(B^*_\sigma \mid \btau_{B_\sigma}) \ge r\big((B^*)' \mid \btau_{B_\sigma}\big) \ge r\big(B^* \mid \btau^*(t_0)\big) - (K + 2)\,,\\[0.2cm]
\rew\big(\tB\big) \ge \rew\big((B^*)'\big) \ge r\big(B^* \mid \btau^*(t_0)\big) - (K+2)\,,
\end{gather*}
which allow to complete the missing parts in the proofs.
Note that in the last equation, the first inequality holds true thanks to the fact that $(B^*)'$ has also two first actions that are different, while the second inequality can be recovered from similar arguments as those used to prove \Cref{lem:key}.
\end{proof}

Similarly to the constant case, we can use a \texttt{CombUCB1} approach to produce a sequence of blocks with small regret against $\tB$.
The only change is the size of the representation: it is now of dimension $2Kd^2$ since we have $2d$ possible states for the arms (ranging from $-d$ to $d$).
Combining Theorem 6 in \citet{kveton2015tight} with \Cref{prop:new}, we obtain \Cref{thm:new}.

\begin{theorem}\label{thm:new}
Let $(\mu_i)_{i=1}^K$ be an LSD bandit with $K$ arms.
Let $T \ge 0$ be the horizon, and choose $d \ge 0$ that divides $T$.
Then \texttt{ISI-CombUCB1}, run with block size $d$ and exploration parameter $\alpha=1.5$, has regret bounded by
\[
R_T \le \frac{(K+2)T}{d} + 47d\sqrt{2KT\log\frac{T}{d}} + \left(\frac{\pi^2}{3} + 1\right)2Kd^3\,.
\]
Choosing $d \propto T^{1/4}$, we obtain $R_T = \tilde{\mathcal{O}}(KT^{3/4})$, where $\tilde{\mathcal{O}}$ is neglecting logarithmic factors.
\end{theorem}


\subsection{Integer Linear Program in the General Case}
\label{sec:negative}

The last step we need to adapt is the Integer Linear Program which constitutes the optimization problem of \texttt{ISI-CombUCB1}.
As explained at the beginning of the section, one difference is that we need to enforce the first two actions of the block to be different for calibration purposes.
The second important difference is the size of the representation: we introduce $Y^+$ and $Y^-$, both of size $Kd^2$, to encode pulls in positive and negative states respectively.
The problem can be described as follows.

Let $F \in \{0, 1\}^{K \times d}$, such that $F[i, t] = 1$ if and only if arm $i$ is played for the first time in the block at time step $t$.
Let $Y^{+} \in \{0, 1\}^{K \times d \times d}$, such that $Y^{+}[i, j, t] = 1$ if and only if arm $i$ is played at time step $t$ in state $j$.
Let $Y^{-} \in \{0, 1\}^{K \times d \times d}$,  such that $Y^{-}[i, j, t] = 1$ if and only if arm $i$ is played at time step $t$ in state $-j$.
In comparison to the previous Integer Linear Program, the objective function, i.e., the sum of the current UCBs for the second actions present in the block, and the conditions (AC), (UFP), and (FPF) remain unchanged.
Their formula are recalled for completeness.
As for the novelties, we introduce the Change Action (CA) constraint, i.e., the second pull in the block must differ from the first.
Time Consistency (TC) is now divided into TC for positive states ($\text{TC}^{+}$), i.e., an arm can be pulled in state $j \ge 1$ at time step $t$ only if it was pulled at time step $t-j$, and not pulled since, and TC for negative states, ($\text{TC}^{-}$), i.e., an arm can be pulled in state $-j \le - 1$ at time step $t$ only if it has been pulled consecutively for the last $j$ time steps.
Note that it is easier to express TC$^-$ using two conditions, depending on whether $-j = -1$ or $-j \le -2$.
Overall, we have (for index limits that make sense)
\begin{align}
&\sum_{i, j, s} \Big(Y^{+}[i, j, s] + Y^{-}[i, j, s]\Big)\,U_t(i, j) \tag{objective}\\[0.2cm]
\forall t \qquad &\sum_i F[i, t] + \sum_{i, j} Y^{+}[i, j, t] + \sum_{i, j} Y^{-}[i, j, t] = 1 \tag{AC}\\[0.1cm]
\forall i \qquad &\sum_t F[i, t] \le 1 \tag{UFP}\\[0.1cm]
\forall i, t \qquad &\sum_j Y^{+}[i, j, t] + \sum_j Y^{-}[i, j, t] \le \sum_{s=1}^{t-1} F[i, s] \tag{FPF}\\[0.1cm]
&\sum_i F[i, 2] = 1 \tag{CA}\\[0.1cm]
\forall i, j, t \qquad &Y^{+}[i, j, t] \le F[i, t-(j+1)] + \sum_l Y^{+}[i, l, t-(j+1)] - \sum_s Y^{+}[i, j-s, t-s]\nonumber\\
&\hspace{4.2cm}+ \sum_l Y^{-}[i, l, t-(j+1)] - \sum_l Y^{-}[i, l, t-j] \tag{$\text{TC}^{+}$}\\[0.1cm]
\forall i, t \qquad &Y^{-}[i, 1, t] \le \sum_l Y^{+}[i, l, t-1] + F[i, t - 1] \tag{$\text{TC}^{-}_1$}\\[0.1cm]
\forall i, t, j \ge 2 \qquad &Y^{-}[i, j, t] \le Y^{-}[i, j-1, t-1]\tag{$\text{TC}^{-}_2$}
\end{align}

Similarly to \Cref{sec:method}, we can approximately solve the above Integer Linear Program by a Branch-and-Bound-like approach, that we detail in the next section.


\subsection{Details about the Branch-and-Bound Heuristic}
\label{sec:heuristic}

In this section, we provide more details about the heuristic we use to approximately solve the Integer Linear Program (ILP).
It works as follows.
For every $i \le K$, we set the first action of the block (of total size $d$) to~$a^{(i)}$.
We then solve a relaxed version of the ILP, optimizing only for actions $a_2, \ldots a_d$ (recall that $a_1$ is fixed to $a^{(i)}$), and allowing for continuous values in $[0, 1]$, instead of $\{0, 1\}$.
This can be done efficiently as the relaxed problem is a standard Linear Program.
We finally set $a_1$ to the $a^{(i)}$ which has given the highest reward according to the relaxed ILP.
We reiterate, by testing values for the second action, and solving the relaxed version with respect to $a_3 \ldots a_d \in [0, 1]^{d-2}$, and so on.
Let $\texttt{LP}$ be the function that takes as input the current UCBs and a fixed block of size $s < d$, and outputs the best continuous solution in $[0, 1]^{d-s}$ for actions $a_{s+1} \ldots a_d$.
Let $\texttt{reward}$ be the function returning the objective value of any sequence (possibly partially continuous).
Our heuristic is summarized in \Cref{alg:rs}.

\begin{algorithm}[!ht]
\SetKwInOut{Input}{input}
\SetKwInOut{Init}{init}
\SetKwInOut{Parameter}{Param}
\caption{Approximate ILP Solver}
\label{alg:rs}
\Input{~Current UCBs $U_t = [U_t(i, j)] \in \mathbb{R}^{K \times 2d}$ for all $i \le K$ and $-d \le j \le d$}\vspace{0.1cm}
\Init{~\texttt{block = []}}\vspace{0.1cm}
\For{$s=1 \ldots d$}{\vspace{0.1cm}
    
    \For{$i=1 \ldots K$}{\vspace{0.1cm}
        $\texttt{block}_\mathrm{\,tmp} = \texttt{block} + [a_i]$\hspace{2.4cm}\tcp{test $a_i$ as next discrete action (current block size:\,$s$)}\vspace{0.1cm}
        
        $\texttt{block}_\mathrm{\,cont} = \texttt{LP}(U_t, \texttt{block}_\mathrm{\,tmp})$\hspace{1.55cm}\tcp{find the best continuous continuation (of size $d - s$)}\vspace{0.1cm}
        
        $r^i = \texttt{reward}(\texttt{block}_\mathrm{\,tmp}, \texttt{block}_\mathrm{\,cont})$\hspace{0.9cm}\tcp{compute the total relaxed reward of the half-discrete}\vspace{0.1cm}
        
        \hspace{6.2cm}\tcp{half-continuous block (of overall size $d$)}
    }
    \vspace{0.15cm}
    
    $i^* = \argmax_{i \le K} ~ r^i$\vspace{0.1cm}
    
    $\texttt{block} = \texttt{block} + [a_{i^*}]$\hspace{3.35cm}\tcp{keep the discrete action with highest relaxed reward}\vspace{0.1cm}
}

\Return{$\mathrm{block}$}
\end{algorithm}

\section{ADDITIONAL EXPERIMENTS}

In this section, we provide additional experiments to elaborate more on the performance of \texttt{ISI-CombUCB1}.
In particular, we enrich the comparison made in \Cref{sec:expe} by benchmarking two new algorithms based on calibration sequences (CS).
These approaches first calibrate the system by playing a permutation $\sigma$ of all the arms, and then play the best block according to the state reached after $\sigma$, see \eqref{eq:pbm_calib}.
CS-based approaches are known to be suboptimal, as they calibrate more arms than necessary, here $5$ ($=K$) instead of $2$ (number of different arms in the optimal block).
In addition, since the calibration phase is of size $K = 5 > 3$, it might prevent from seeing the spike of arm $a^{(1)}$ at $\tau=3$.
We therefore benchmark two permutations, $\texttt{CS-worst} = [a^{(1)}, a^{(2)}, a^{(3)}, a^{(4)}, a^{(5)}]$, and $\texttt{CS-best} = [a^{(5)}, a^{(4)}, a^{(3)}, a^{(2)}, a^{(1)}]$, named depending on whether they allow to see the spike or not.
Again, \Cref{alg:rs} is used as an inner subroutine to approximately compute the solution to the optimization problem~\eqref{eq:pbm_calib}.
Results are shown in \Cref{fig:add_exp_calibseq}.
\medskip

One may argue that the example of \Cref{sec:expe} is unfair to \texttt{OracleGreedy}, because the algorithm is tricked by phenomena that occur beyond the block size $d$, see \Cref{fig:exp_figures} (top).
We now present an example where it is not the case.
Consider the $2$-armed LSD bandit with Bernoulli rewards and
\[
\mu_1(\tau) = \begin{cases}0.06 &\text{if } \tau = 1\\0.95 &\text{if } \tau \ge 2\end{cases} \text{\qquad and \qquad} \mu_2(\tau) = 0.05 \quad \forall \tau\,.
\]
The reward functions are represented in \Cref{fig:add_exp_foolOG} (left).
Here, \texttt{OracleGreedy} constantly pulls arm $a^{(1)}$, for an average reward of $0.06$.
Instead, \texttt{ISI-CombUCB1} and \texttt{CombUCB1} alternate between arms $a^{(1)}$ and $a^{(2)}$, for an optimal average reward of $0.5$.
Note that \texttt{ISI-CombUCB1}, as opposed to \texttt{CombUCB1}, requires a calibration pull before playing this optimal block.
The latter bringing little reward, it slightly degrades the average reward and explains the small difference in performance seen in the plots.
However, besides the calibration, both algorithms converge towards the same block, which is optimal.
Results for $d=10$ are reported in \Cref{fig:add_exp_foolOG} (right).

In conclusion, calibration is an operation that might occasionally degrade the performance (in a controlled way), but that on the other hand guarantees to avoid risky decoys, such as the one discussed in \Cref{sec:expe}.
This behavior is in line with the fact that the regret of \texttt{ISI-CombUCB1} is well understood theoretically, while \texttt{CombUCB1} is hard to analyze due to the interferences.

The code is written in Python, and uses the \texttt{CVXOPT} package \citep{andersen2013cvxopt} to solve the Linear Programs.
It is publicly available at the following GitHub repository: \href{https://github.com/GiuliaClerici/LSD_bandits.git}{GiuliaClerici/LSD\_bandits}.

\begin{figure}[!t]
    \centering
    \includegraphics[width=.43\linewidth, height=6cm]{fig/arm_rwdfunc}
    \qquad
    \includegraphics[width=.4\linewidth, height=5.9cm]{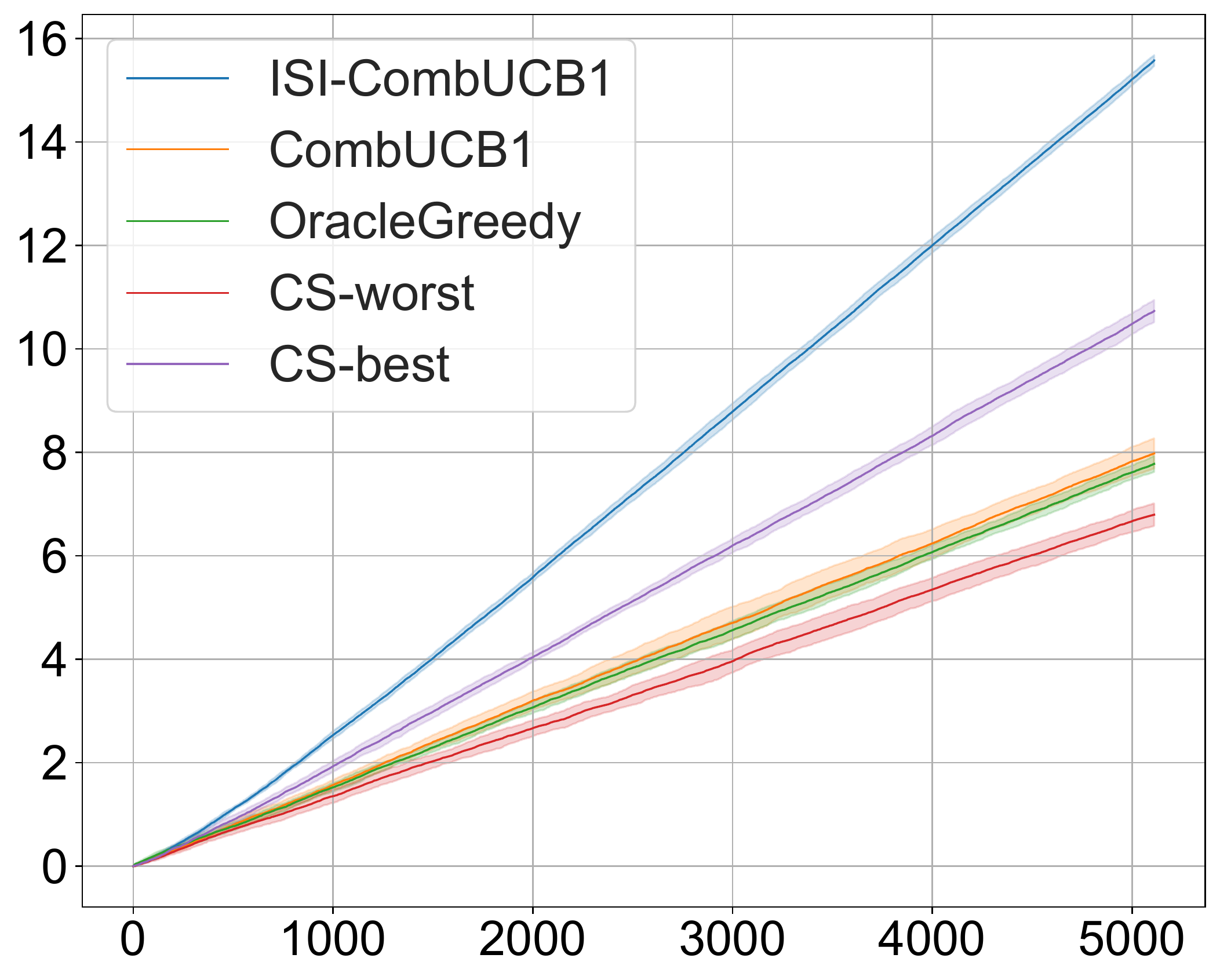}
    \caption{Reward functions with respect to $\tau$ (left) and cumulative rewards in thousands (right).}
    \label{fig:add_exp_calibseq}
\end{figure}

\begin{figure}[!t]
    \centering
    \includegraphics[width=.43\linewidth, height=6cm]{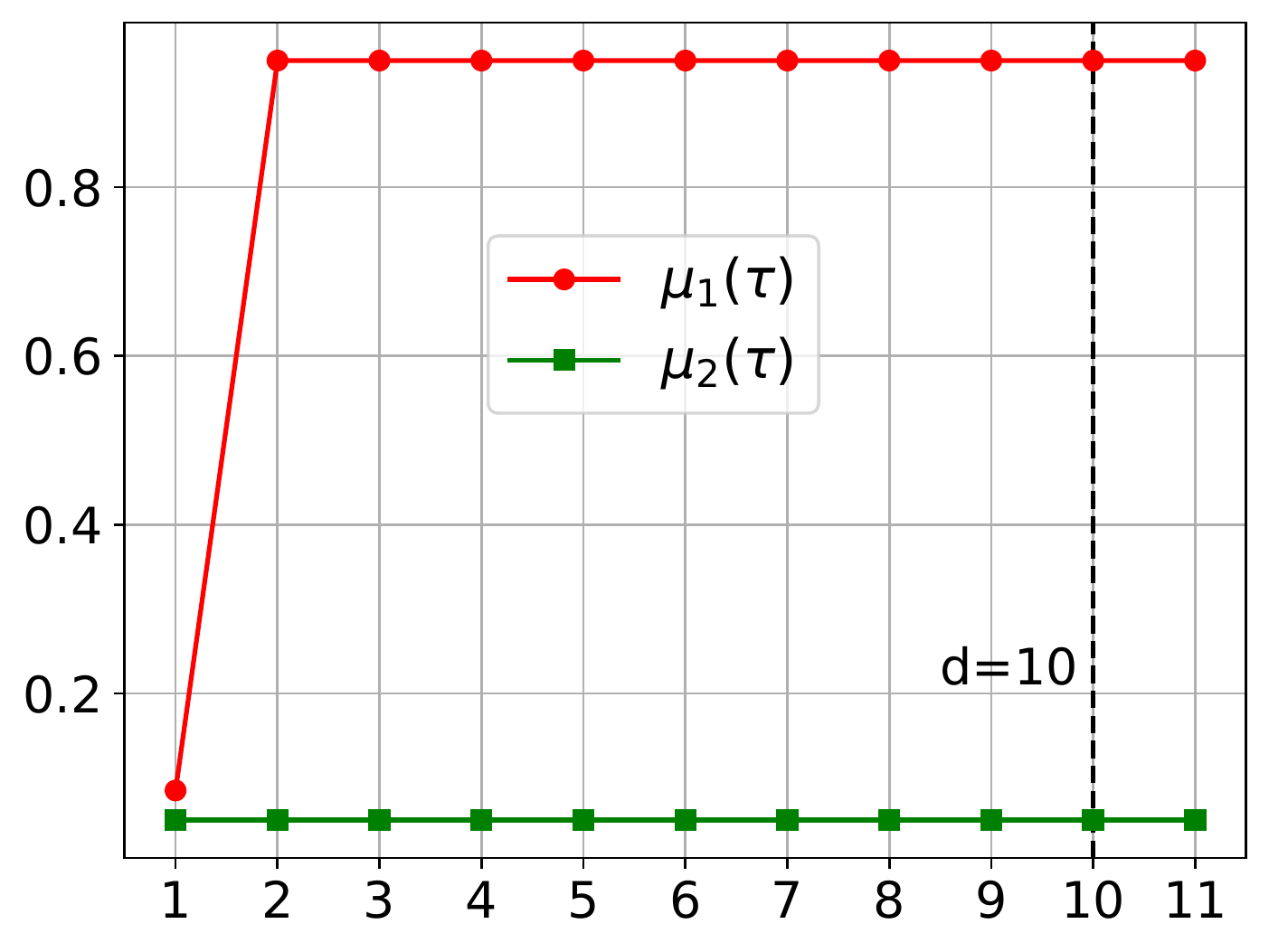}
    \qquad
    \includegraphics[width=.4\linewidth, height=6cm]{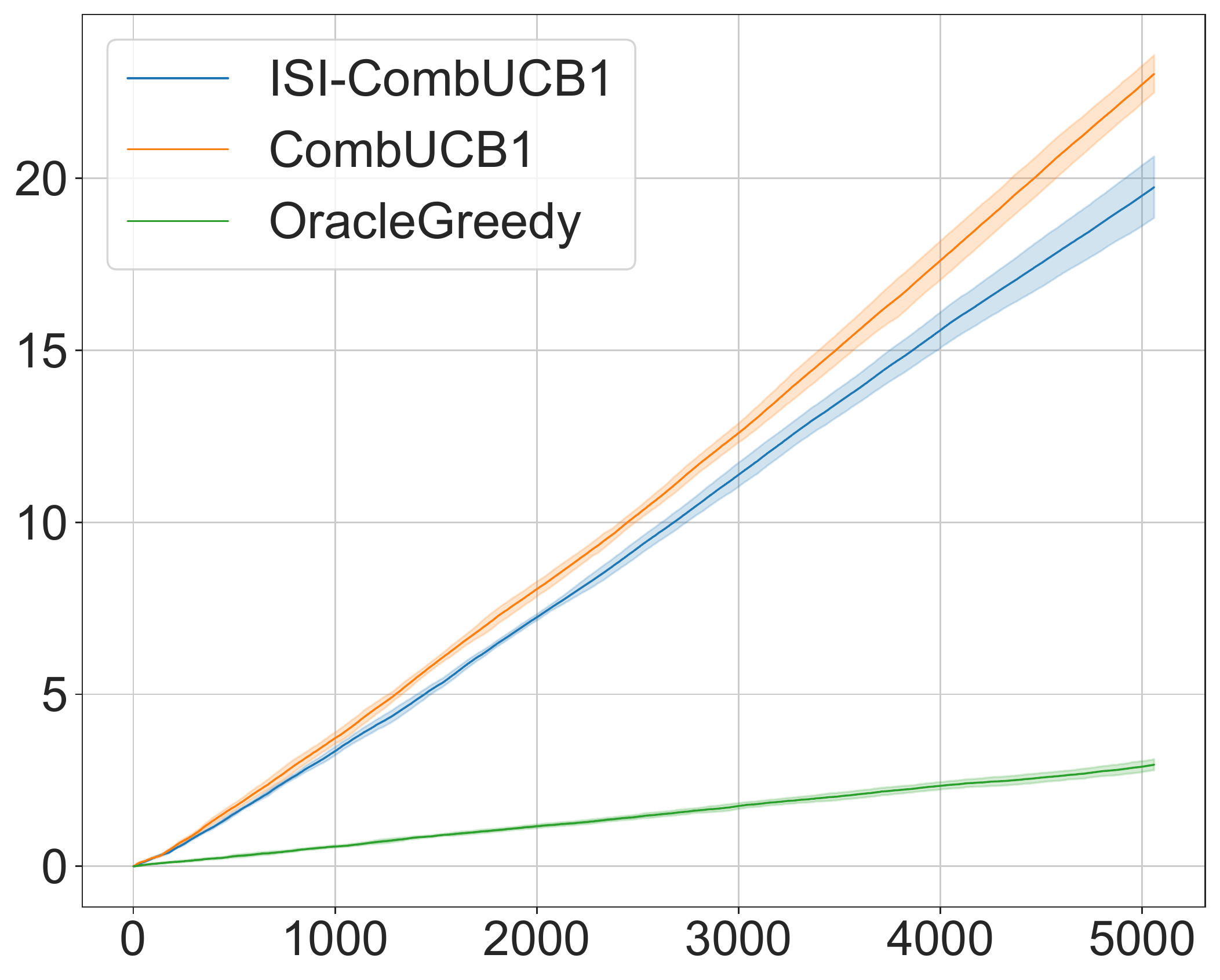}
    \caption{Reward functions with respect to $\tau$ (left) and cumulative rewards in thousands (right).}
    \label{fig:add_exp_foolOG}
\end{figure}

\end{document}